\documentclass{article}

\usepackage[utf8]{inputenc}
\usepackage[T1]{fontenc}
\usepackage[preprint]{neurips_2026}

\usepackage{url}
\usepackage{booktabs}
\usepackage{amsfonts}
\usepackage{amsmath}
\usepackage{amssymb}
\usepackage{nicefrac}
\usepackage{microtype}
\usepackage{xcolor}
\usepackage{graphicx}
\usepackage{wrapfig}
\usepackage{placeins}
\usepackage{tabularx}
\usepackage{enumitem}
\usepackage{listings}
\usepackage{caption}
\usepackage{subcaption}
\usepackage{amsthm}
\usepackage{hyperref}

\hypersetup{colorlinks=true,linkcolor=blue,urlcolor=blue,citecolor=blue,
  pdftitle={Fidelity Probes for Specification--Code Alignment},
  pdfauthor={Ferhat Erata, Hao Zhou, Jun Huan}}

\newtheorem{theorem}{Theorem}[section]
\newtheorem{proposition}[theorem]{Proposition}

\newtheorem{corollary}[theorem]{Corollary}
\theoremstyle{definition}
\newtheorem{definition}[theorem]{Definition}
\newtheorem{assumption}[theorem]{Assumption}
\theoremstyle{remark}

\definecolor{addblue}{RGB}{37, 99, 235}
\definecolor{fixorange}{RGB}{234, 88, 12}
\definecolor{removered}{RGB}{220, 38, 38}
\newcommand{\ADD}{\textcolor{addblue}{\textsc{Add}}}
\newcommand{\FIX}{\textcolor{fixorange}{\textsc{Fix}}}
\newcommand{\REMOVE}{\textcolor{removered}{\textsc{Remove}}}

\definecolor{cblkw}{RGB}{30, 64, 175}
\definecolor{cblnum}{RGB}{14, 116, 144}
\definecolor{cblstr}{RGB}{21, 128, 61}
\definecolor{cblcomment}{RGB}{100, 116, 139}
\lstdefinelanguage{CobolMotiv}{
  morekeywords={IDENTIFICATION,DIVISION,PROGRAM-ID,DATA,LINKAGE,SECTION,
    PIC,PROCEDURE,USING,EVALUATE,TRUE,WHEN,OTHER,MOVE,TO,IF,ADD,END-EVALUATE,
    END-IF,GOBACK},
  sensitive=true,
  morecomment=[l][\color{cblcomment}\itshape]{*},
  morestring=[b]',
}
\lstdefinestyle{calcdisc}{
  language=CobolMotiv,
  basicstyle=\ttfamily\scriptsize\color{black!85},
  keywordstyle=\color{cblkw}\bfseries,
  stringstyle=\color{cblstr},
  commentstyle=\color{cblcomment}\itshape,
  numbers=left,
  numberstyle=\ttfamily\tiny\color{black!40},
  numbersep=6pt,
  frame=none,
  xleftmargin=14pt,
  xrightmargin=0pt,
  aboveskip=0pt,
  belowskip=0pt,
  breaklines=false,
  columns=fullflexible,
  keepspaces=true,
  literate={>}{{$>$}}1 {'}{{\textquotesingle}}1,
}
\definecolor{reqidbg}{RGB}{96, 139, 211}
\newcommand{\reqid}[1]{\begingroup\setlength{\fboxsep}{2pt}%
  \colorbox{reqidbg}{\color{white}\scriptsize\ttfamily\bfseries #1}\endgroup}
\newcommand{\reqshall}{\textbf{shall}}
\newcommand{\reqcond}[1]{\textit{#1}}

\title{Fidelity Probes for Specification--Code Alignment}

\author{%
  Ferhat Erata\thanks{Corresponding author: \href{mailto:ferhat@ieee.org}{\texttt{ferhat@ieee.org}}} \quad Hao Zhou \quad Jun Huan \\
  AWS AI Labs
}
\date{}

\begin{document}

\maketitle

\begin{abstract}
Software modernization often relies on natural-language specifications recovered
from an existing implementation. Incorrect or missing requirements can
lead to defects in the modernized system.
We introduce \emph{fidelity probes} to check a specification against the code
and guide its revision. Each probe pairs a question about program behaviour
with a reference answer derived from the code. A language model answers the
question using only the specification, and the fraction of agreeing answers
defines \emph{fidelity}. Disagreements flag conflicting claims or missing
information and guide proposed corrections and additions to the requirements.
An LLM generates probes directly from code or phrases facts selected through
control-flow, data-flow, and call-graph analysis.
We use an observability rule to focus probe questions on behaviour that the
modernized system is intended to preserve, including user-visible outputs,
changes to stored business data, and interactions with other systems.
We revise the specification using fresh probes and evaluate each version on a
fixed held-out set. A repair--regression model describes how new failures can
offset the gains from repairs.
On CardDemo, held-out fidelity improves from 0.59 to 0.89, exceeding free-form
revision from source code. We also evaluate the method on industrial systems
and public documentation, and assess probe quality and reported defects through
human reviews and comparison with an independent requirements audit.
\end{abstract}

\begin{figure}[t]
  \centering
  \scriptsize
  \setlength{\tabcolsep}{2pt}
  \renewcommand{\arraystretch}{1.0}
  \begin{tabular}{@{}p{0.35\linewidth}@{\hspace{2pt}}p{0.35\linewidth}@{\hspace{2pt}}p{0.28\linewidth}@{}}
    \begin{minipage}[t]{\linewidth}
      \textbf{\footnotesize COBOL source (1/2)}\par\vspace{3.5pt}
      \hrule height 0.4pt\vspace{3pt}
      \begin{lstlisting}[style=calcdisc,firstnumber=1,basicstyle=\ttfamily\fontsize{6pt}{7pt}\selectfont\color{black!85}]
IDENTIFICATION DIVISION.
PROGRAM-ID. CALCDISC.
DATA DIVISION.
WORKING-STORAGE SECTION.
01 WS-BASE-PCT  PIC 9(3).
LINKAGE SECTION.
01 LS-AMOUNT        PIC 9(7)V99.
01 LS-DISCOUNT      PIC 9(3).
01 LS-CUSTOMER-TYPE PIC X(10).
PROCEDURE DIVISION USING LS-AMOUNT
    LS-DISCOUNT LS-CUSTOMER-TYPE.
MAIN-PARA.
  PERFORM 1000-CALC-BASE-TIER.
  PERFORM 2000-APPLY-PREMIUM.
  CALL 'AUDITDB' USING LS-DISCOUNT.
  GOBACK.
\end{lstlisting}
    \end{minipage}
     &
    \begin{minipage}[t]{\linewidth}
      \textbf{\footnotesize COBOL source (2/2)}\par\vspace{3.5pt}
      \hrule height 0.4pt\vspace{3pt}
      \begin{lstlisting}[style=calcdisc,firstnumber=17,basicstyle=\ttfamily\fontsize{6pt}{7pt}\selectfont\color{black!85}]
1000-CALC-BASE-TIER.
  EVALUATE TRUE
    WHEN LS-AMOUNT > 1000
      MOVE 20 TO WS-BASE-PCT
    WHEN LS-AMOUNT > 500
      MOVE 15 TO WS-BASE-PCT
    WHEN OTHER
      MOVE 5  TO WS-BASE-PCT
  END-EVALUATE.
  MOVE WS-BASE-PCT TO LS-DISCOUNT.

2000-APPLY-PREMIUM.
  IF LS-CUSTOMER-TYPE = 'PREMIUM'
    ADD 5 TO LS-DISCOUNT
  END-IF.
\end{lstlisting}
    \end{minipage}
     &
    \begin{minipage}[t]{\linewidth}
      \textbf{\footnotesize Behavioural requirements}\par\vspace{3.5pt}
      \hrule height 0.4pt\vspace{3pt}
      \raggedright\fontsize{6pt}{7pt}\selectfont
      \renewcommand{\reqid}[1]{\begingroup\setlength{\fboxsep}{1.5pt}%
        \colorbox{reqidbg}{\color{white}\fontsize{5.5pt}{6pt}\selectfont\ttfamily\bfseries #1}\endgroup}
      \reqid{REQ-DISC-001}\
      \textbf{When} the order amount is \reqcond{strictly greater than 1\,000}, \textbf{the system} \reqshall\ apply a base discount of 20\%.\par\smallskip

      \reqid{REQ-DISC-002}\
      \textbf{When} the order amount is \reqcond{strictly greater than 500 and at most 1\,000}, \textbf{the system} \reqshall\ apply a base discount of 15\%.\par\smallskip

      \reqid{REQ-DISC-003}\
      \textbf{When} the order amount is \reqcond{500 or less}, \textbf{the system} \reqshall\ apply a base discount of 5\%.\par\smallskip

      \reqid{REQ-DISC-004}\
      \textbf{When} the customer placing the order is a premium member, \textbf{the system} \reqshall\ add an additional 5 percentage points on top of whichever base discount applies under REQ-DISC-001--003.\par\smallskip

      \reqid{REQ-AUDIT-001}\
      Every computed discount \reqshall\ be forwarded to the audit subsystem.
    \end{minipage}
    \\
  \end{tabular}

  \vspace{6pt}
  \begin{tabular}{@{}p{0.40\linewidth}@{\hspace{2pt}}p{0.30\linewidth}@{\hspace{2pt}}p{0.28\linewidth}@{}}
    \begin{minipage}[t]{\linewidth}
      \centering
      \textbf{\footnotesize Control-flow graph (CFG)}\par\vspace{2pt}
      \hrule height 0.4pt\vspace{3pt}
      \includegraphics[width=\linewidth]{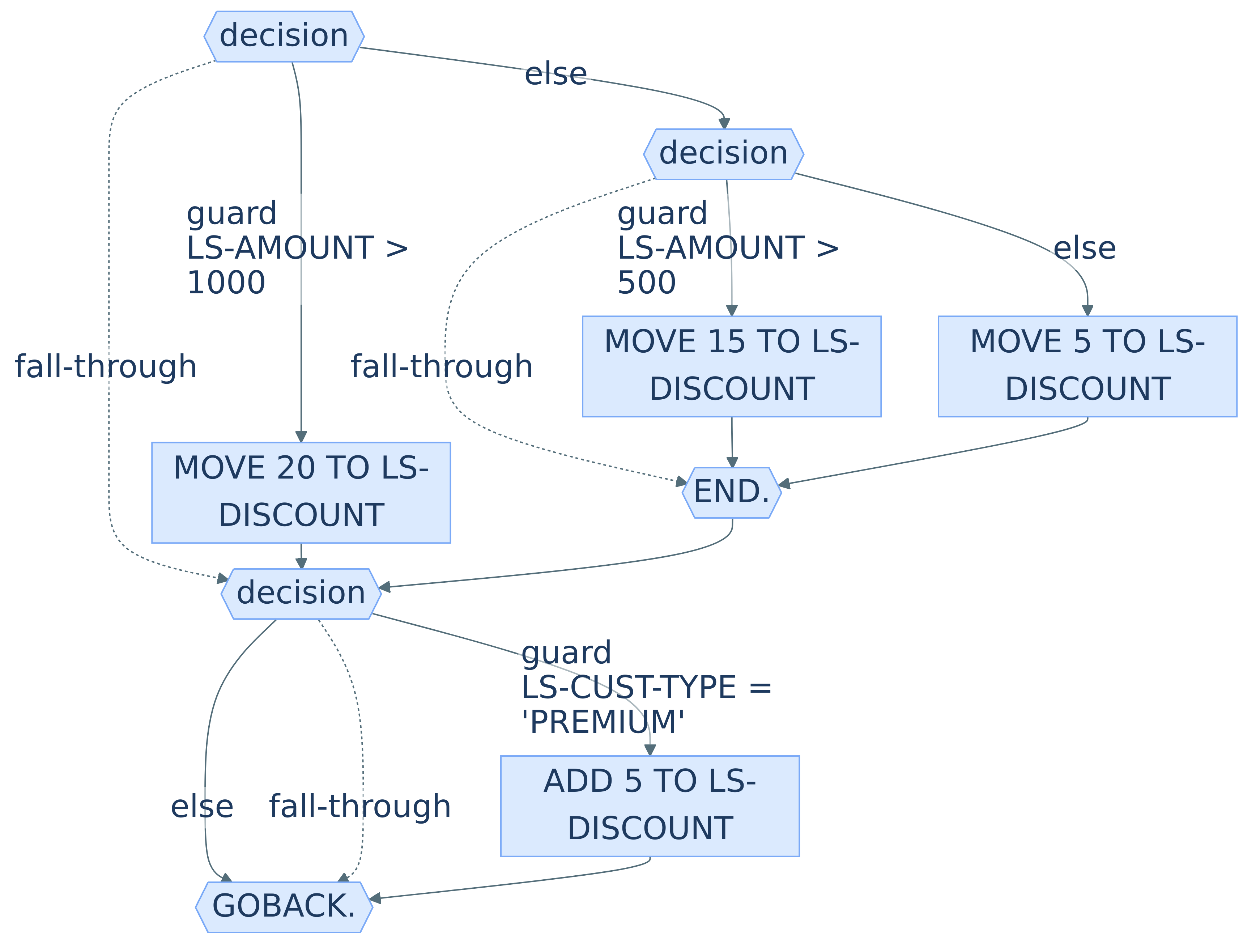}
    \end{minipage}
     &
    \begin{minipage}[t]{\linewidth}
      \centering
      \textbf{\footnotesize Data-flow graph (DFG)}\par\vspace{2pt}
      \hrule height 0.4pt\vspace{3pt}
      \includegraphics[width=\linewidth]{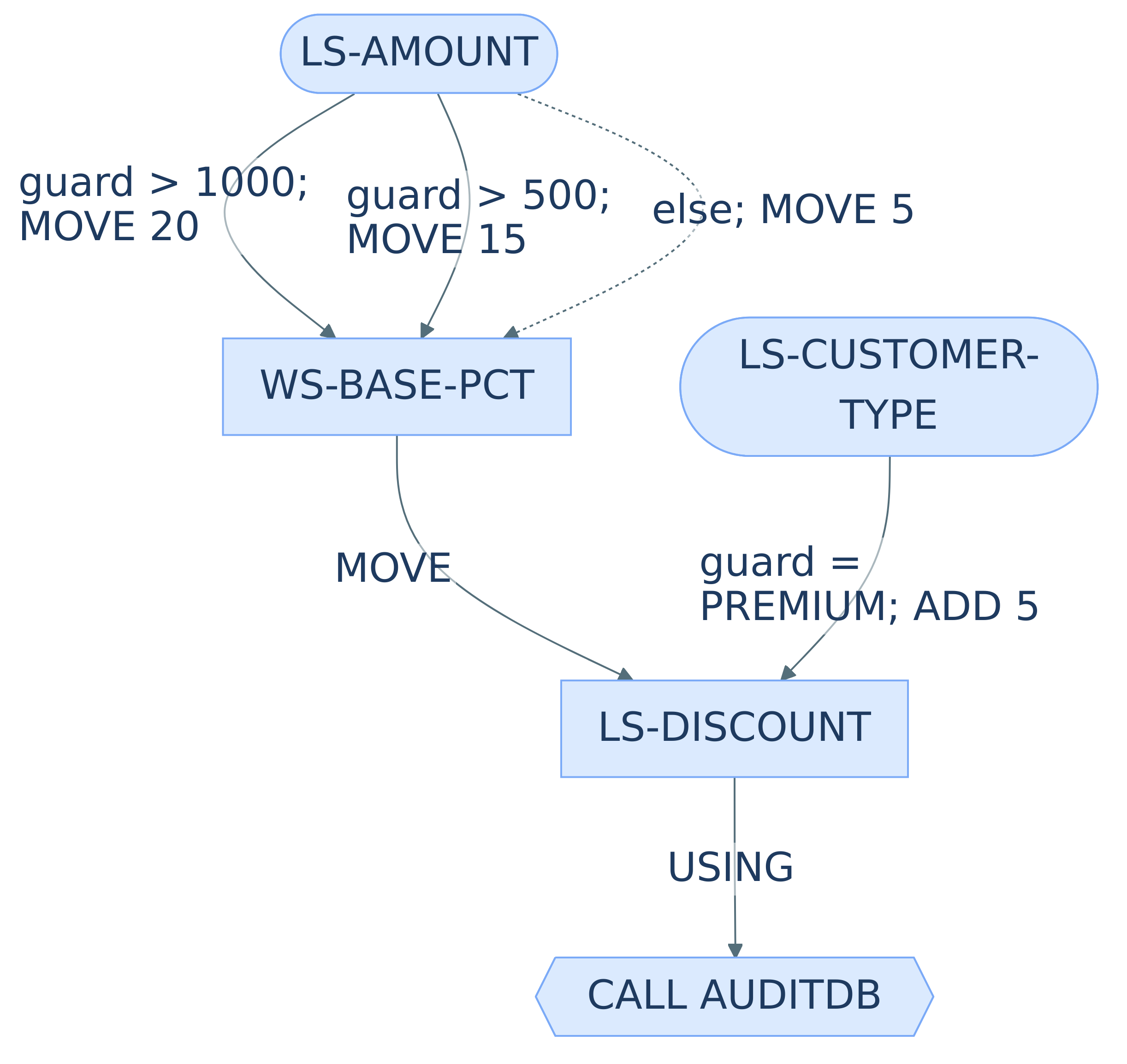}
    \end{minipage}
     &
    \begin{minipage}[t]{\linewidth}
      \centering
      \textbf{\footnotesize { Call graph (CallGraph)}}\par\vspace{2pt}
      \hrule height 0.4pt\vspace{3pt}
      \includegraphics[width=\linewidth]{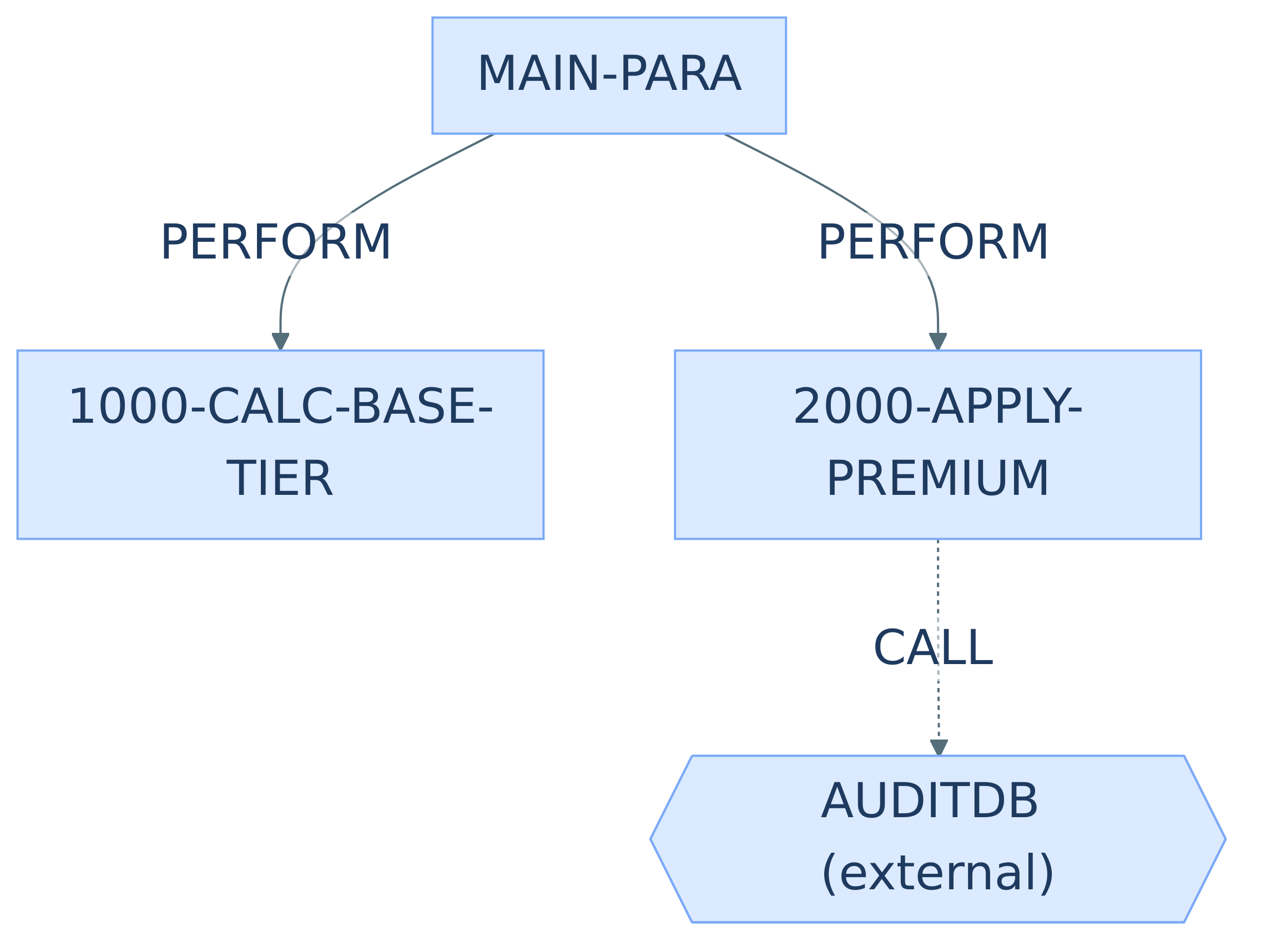}
    \end{minipage}
    \\
  \end{tabular}
  \caption{Motivating example: \texttt{CALCDISC} COBOL source (top, two columns), a candidate behavioural-requirements document, and the three static-analysis projections (CFG, DFG, call graph) the symbolic channels read as ground truth.}
  \label{fig:motivating-calcdisc-v2}
\end{figure}

\section{Introduction}
Modernizing a legacy system often requires a specification of the behaviour
that must be retained in the modernized system. Recovering business rules
from the existing code provides a basis for redevelopment
\citep{sneed2020businessrules}, and the resulting requirements guide both
implementation and acceptance testing. If these requirements misstate a
business rule, the same mistake may be carried
into both the modernized software and the tests used to assess it. Because
the specification is written in natural language, we cannot check its
agreement with the code by executing the document.

Question-answering methods evaluate summaries by asking the same questions
of a summary and its source, then comparing the answers
\citep{wang2020qags,scialom2021questeval}.
We apply this idea to code and specifications, using answers derived from
the code as the reference. Comparing these with answers drawn from the
specification helps identify requirements that may need revision.

We introduce \emph{fidelity probes} to evaluate natural-language
specifications and documentation, whether written manually or generated
by an agentic pipeline.
Each probe pairs a question about program behaviour with an
answer derived from the code. A language model answers the question from
the specification alone, and a separate model call compares its answer
with the reference.
We distinguish agreement, a contradiction when the document makes a
conflicting claim, and a coverage gap when it provides no answer.
The fraction of agreeing probes gives the fidelity score. Contradictions
and gaps provide evidence for proposed edits to the requirements.

Figure~\ref{fig:motivating-calcdisc-v2} illustrates the setting. The code $A$
defines a discount policy, and the candidate specification $B$ expresses it
as EARS-style requirements with stable identifiers \citep{mavin2009ears}.
Consider a probe asking which base discount applies to an order of exactly
1\,000. The code gives 15\%. If a specification instead assigns 20\%,
the disagreement identifies a boundary error in the requirement.
A separate probe can ask how premium status affects the discount.
If the document does not describe the premium increment, the unanswered
question identifies a coverage gap. The two findings lead to different edits:
correcting the threshold in the first case and adding the missing rule in
the second.

Fidelity probes supply ``another level of indirection''
\citep{lampson2021hints}: a question and reference answer make a behavioural
claim checkable, much as a test pairs a case with an expected result.
Questions in the application's vocabulary let business analysts review
these claims, while source references support technical inspection.
The observability rule permits abstraction: platform mechanics matter only
when they affect the behaviour in scope.

A language model revises the specification using findings from fresh probes.
A fixed held-out set and judging procedure provide a common basis for
comparing documents and revisions. The reviser never receives held-out
questions or verdicts, and model parameters remain fixed.

We compare the loop with single-shot and fixed-probe variants and a
free-form baseline that adapts iterative self-refinement \citep{madaan2023self} to
revision from source code (Section~\ref{sec:baselines}).
CardDemo \citep{awsCardDemo} is the main benchmark, with additional studies
on industrial systems and public documentation.
Our repair--regression analysis explains plateaus through the balance
between repairs and newly introduced failures
(Section~\ref{sec:brief-formal}).

\section{Related work}
\textit{Verification and testing.} Formal methods such as Dafny, Lean, and
TLA\textsuperscript{+} reason about machine-interpretable specifications
\citep{leino2010dafny,moura2021lean,lamport2002tla}.
Our setting involves natural-language requirements produced before a
replacement is implemented. Test-generation methods
\citep{claessen2000quickcheck,godefroid2005dart,sen2005cute,pacheco2007randoop,fraser2011evosuite}
check behaviour by executing programs. We use questions to assess whether
the requirements describe the behaviour of the reference implementation.

\textit{Refinement and specifications.} Self-Refine, Reflexion, ReAct, and
Self-Debugging use feedback to improve model outputs
\citep{madaan2023self,shinn2023reflexion,yao2022react,chen2023selfdebug}.
We adapt self-refinement to requirements as a free-form baseline and compare
it with revision guided by probe failures. Clover, Wybe, and specification synthesis
methods connect code generation to formal checks
\citep{sun2024clover,gloeckle2026wybecoder,ma2024specgen,wen2024autospec,wu2023lemur,ye2025verina}.
Our target is a natural-language document, and the feedback is a sampled
agreement score with evidence for individual edits.

\textit{Grounding and evaluation.} The graph channels build on program slicing
and dependence analysis
\citep{weiser1984slicing,ferrante1987pdg,horwitz1990sdg,reps1995ifds}.
We attach source locations and requirement identifiers to findings so that
reviewers can trace their evidence.
Models still phrase the questions and compare the answers, and errors can
arise in either step. LLM judges can also introduce systematic biases
\citep{zheng2023judging,wang2023fair,liu2023geval}.
We compare results across models and use human review to check the probes
and the findings they produce.

\section{Fidelity probes}
\label{sec:brief-framework}\label{sec:brief-setup}
Let $A$ be the reference implementation and $B$ the candidate specification.
We measure whether $B$ describes the behaviour of $A$.
Intended changes to legacy behaviour must be recorded separately.
A generator returns probes
$Q(A)=\{(q_i,y_i,\kappa_i)\}_{i=1}^n$, with a question $q_i$, reference answer
$y_i$, and behaviour category $\kappa_i$. Each probe also records a source
location or extracted fact so that its reference answer can be inspected.
To score the probe, an LLM answers from $B$ alone, returning
$J(B,q_i)=(\hat y_i,e_i,c_i)$, with an answer $\hat y_i$, a supporting
passage $e_i$ in $B$, and a support flag $c_i$.
This call sees neither the code nor $y_i$.
The flag is \emph{supported} when the document supplies enough information
to answer, and \emph{uncertain} when that information is absent or ambiguous.

\textbf{Fidelity and actions.}\label{sec:brief-fidelity}
A separate comparator receives the question and both answers, then returns
a Boolean agreement decision. We write $\hat y_i\equiv y_i$ when it accepts
the pair. Its prompt allows differences in wording and some differences in
specificity: an exact threshold and a general statement that the limit is
validated can agree (Appendix~\ref{app:prompt:comparator}).
Thus an omitted detail need not cause disagreement.
Let $T$ denote the indices of the probes being evaluated.
When the comparator rejects a pair, the support flag distinguishes a
conflicting answer from missing or ambiguous information:
\begin{equation}
 v(q_i)=\begin{cases}
 \mathrm{agree}&\hat y_i\equiv y_i,\\
 \mathrm{contradict}&\hat y_i\not\equiv y_i,\ c_i=\mathrm{supported},\\
 \mathrm{gap}&\hat y_i\not\equiv y_i,\ c_i=\mathrm{uncertain},
 \end{cases}
 \quad F(A,B;T)=\frac1{|T|}\sum_{i\in T}\mathbf1[v(q_i)=\mathrm{agree}].
\end{equation}
The contradiction and gap rates $C$ and $G$ satisfy $F+C+G=1$.
The verdicts suggest different edits, but all three are model judgments:
an incorrect reference answer, specification answer, or comparison can change
the score.

\label{sec:brief-actions}
Gaps trigger \ADD\ to supply or clarify behaviour. Contradictions trigger \FIX\
to correct a claim, or \REMOVE\ when the specification asserts behaviour
unsupported by the reference. An action
$(\mathbf a_i,\rho_i,\gamma_i,\epsilon_i)$ contains its type, requirement
identifiers, editing guidance, and evidence. At iteration $k$, the reviser
uses this action set:
\begin{equation}
 \mathcal A_k=\{(\mathbf a_i,\rho_i,\gamma_i,\epsilon_i):
 v(q_i)\ne\mathrm{agree}\},\qquad B_{k+1}=R(B_k,\mathcal A_k).\label{eq:brief-revision}
\end{equation}
\label{sec:brief-iter}
We ask the reviser to make local edits and preserve requirement identifiers.
Because it returns the entire document, it can also change passages outside
the requested edits and introduce regressions
(Section~\ref{sec:brief-formal}).
Figure~\ref{fig:brief-loop} shows the loop; Appendix~\ref{app:brief-repair-examples-symbolic}
follows individual edits.

\textbf{Curriculum and held-out questions.}\label{sec:brief-traintest}
The \emph{curriculum} is the set of questions whose failures guide a revision.
We sample a fresh curriculum $T_{\rm train}^{(k)}$ for each $B_k$ and keep a
separate set $T_{\rm test}$ fixed throughout. Held-out failures are never
passed to the reviser. Both scores at step $k$ describe the same document
$B_k$ before its next revision, so their difference
$\Delta_k=F_k^{\rm train}-F_k^{\rm test}$ compares agreement rates on the
two question sets for that version of the specification.
The loop produces specifications $B_0,\ldots,B_i$, with subscripts indicating
revision order: $B_0$ is the initial specification and $B_i$ is the
specification after $i$ revisions. It stops after a fixed revision budget
or when improvement in fidelity falls below a chosen threshold, and returns
$B_i$ without using held-out scores to select a checkpoint.
A curriculum--test gap can reflect fitting the curriculum, different question
coverage, or judgment noise; it does not uniquely identify overfitting.

\textbf{Probe distribution.}\label{sec:brief-probe-dist}
The direct channel asks an LLM to generate questions and reference answers
while reading $A$. The symbolic route first selects a fact from control-flow,
data-flow, or call-graph analysis. The \emph{informalizer} uses an LLM to phrase
each selected fact in business language, producing a question and reference answer.
Table~\ref{tab:brief-channels} summarizes the questions each channel asks;
Appendix~\ref{app:channel-construction} describes graph construction,
the emitted fact types, sampling, and observability filters.

The weights set each source's share. The graph channels are indexed by
$j\in\{\mathrm{cfg},\mathrm{dfg},\mathrm{cg}\}$, with weights
$\beta_j\ge0$ and $\sum_j\beta_j=1$. For $\alpha\in[0,1]$, the mixture is
\begin{equation}
\mathcal D_A(\alpha,\beta)=\alpha\mathcal D_A^{\rm llm}
+(1-\alpha)\sum_j
\beta_j\mathcal D^j.
\end{equation}
Thus $\alpha=1$ uses direct generation, $\alpha=0$ uses graph facts, and
intermediate values combine them. The graph weights $\beta_j$ allocate the
symbolic share among the three channels. Each fact type is assigned to one
channel. Several facts can describe the same behaviour, so probes from
different channels may overlap in what they cover.
Within each graph channel the sampler shuffles facts and selects without
replacement, allowing at most one fact per region. These restrictions create
dependence between samples.

\begin{table}[!htb]
  \centering
  \footnotesize
  \setlength{\tabcolsep}{3pt}
  \renewcommand{\arraystretch}{1.15}
  \caption{Symbolic probe channels and example questions. The LLM informalizer phrases each fact in business language, producing a question and reference answer.}
  \label{tab:brief-channels}
  \begin{tabularx}{\linewidth}{@{}lX@{}}
    \toprule
    \textbf{Symbolic channel} & \textbf{Example question} \\
    \midrule
    $\mathcal{D}^{\mathrm{cfg}}$ (control-flow graph) & \emph{What observable effect follows when condition $X$ holds?} \textbf{(guard)} \\
    $\mathcal{D}^{\mathrm{dfg}}$ (data-flow graph) & \emph{Which input, under what transformation, produces observable output $Y$?} \textbf{(data)} \\
    $\mathcal{D}^{\mathrm{cg}}$ (call graph) & \emph{After event $X$, what happens next (which screen, output, or handoff)?} \textbf{(flow)} \\
    \bottomrule
  \end{tabularx}
\end{table}

\begin{definition}[Behaviourally observable probe]\label{def:observability}
A fidelity probe is behaviourally observable if it asks about business
functionality that a user can observe through system outputs, changes to
stored business data, or interactions with other systems. The probe may
concern an effect itself or the business rules and conditions that determine
it. Internal structures and platform mechanisms are relevant only when they
affect that functionality.
\end{definition}
We implement this rule through fixed prompt instructions and filters on
graph facts; these do not determine
which observable behaviours a particular modernization project must retain.
For example, a migration may redesign a payment interface while preserving
the rules for updating balances and accepting payments. Probes for that
migration would ask how a payment changes the balance and under which
conditions it is accepted. Questions about interface behaviour depend on
what the modernized system is intended to preserve.

The informalizer uses fixed prompts and temperature zero, but its wording can
still change the meaning of a fact. The template comparison examines how phrasing affects
the score, while the human reviews assess correctness and whether the probes
cover behaviour relevant to the modernized system
(Section~\ref{sec:human-validation}, Appendix~\ref{app:template-study}).

\section{Analysis of revision progress}
\label{sec:brief-formal}
Each revision can repair some probe failures while introducing others.
We analyse these transitions to explain how fidelity changes and why
improvement can plateau. Section~\ref{sec:brief-pure-llm} measures the repair
and regression rates on a fixed held-out set and uses them to interpret
the revision trajectory.

Let $\pi_k$ be the fraction of failing probes repaired in a revision and $r_k$
the fraction of agreeing probes that regress, both measured on the same fixed
probe set. The next score combines probes that remain in agreement with
those that move from failure to agreement:
\begin{equation}
 F_{k+1}=(1-r_k)F_k+\pi_k(1-F_k),\qquad
 F_{k+1}-F_k=\pi_k(1-F_k)-r_kF_k.
 \label{eq:iclr-transition}
\end{equation}
For $\pi_k+r_k>0$, $F_k^\dagger=\pi_k/(\pi_k+r_k)$ is the fidelity level
where gains from repairs balance losses from regressions. Progress is
non-negative precisely when $F_k\le F_k^\dagger$. For constant rates,
we write this level as $F^\dagger$.
\begin{theorem}[Constant-rate convergence]\label{thm:fixedpoint}
If the exact recurrence has constant rates $\pi,r$, with $\pi+r>0$ and
$|1-\pi-r|<1$, then
$F_k-F^\dagger=(1-\pi-r)^k(F_0-F^\dagger)$ and
$F_k\to F^\dagger=\pi/(\pi+r)$.
\end{theorem}
\noindent Substituting rates measured for a completed revision reproduces its observed
score. Forecasting a later revision requires an additional assumption that
the rates will remain stable.
Appendix~\ref{app:proofs} provides proofs and a separate analysis of sampling
uncertainty.

\section{Evaluation}
\label{sec:brief-empirical}
Our main benchmark is AWS CardDemo \citep{awsCardDemo}, a public credit-card
application. The main revision experiment covers 15 COBOL programs, about
12\,k lines, with curated EARS-style specifications averaging 127 requirements
with stable identifiers per program. The programs map to eleven functional
requirements documents; some share documents (Appendix~\ref{app:brief-markov}).
The main run uses the direct LLM channel, with Opus~4.6 as generator, judge,
and reviser at temperature zero. We did not repeat these comparisons across
random seeds.

Additional studies use two proprietary mainframe systems, System A
(COBOL, about 480\,k lines including copybooks) and System B
(PL/1, about 57\,k lines). We also evaluate Redis command documentation
against six C modules (about 6.7\,k lines) and PetClinic specifications
against 25 Java files (about 1.8\,k lines). An independent audit of
System B's requirements supplies external human judgments.

\subsection{CardDemo revision trajectory}
\label{sec:brief-pure-llm}
On all 15 programs and the same 911 held-out probes, pooled fidelity rises from
0.587 to 0.885. Contradictions fall from 120 to 46 and gaps from 256 to 59.
The first revision accounts for most of the gain.
Figure~\ref{fig:brief-dashboard} shows program trajectories.

The core-cohort summaries in Table~\ref{tab:brief-pi-rhat} and
Appendix~\ref{app:brief-markov} exclude \textsc{CBACT01C} and \textsc{CORPT00C},
whose behaviour depends on cross-program context. On this 13-program core,
program-mean fidelity rises from 0.627 to 0.939, and pooled fidelity from
0.590 to 0.932 on 785 probes. The dashboard reports program means; baseline
comparisons pool all 15 programs. Figure~\ref{fig:brief-dashboard}'s error
categories can overlap; the counts above use disjoint probe verdicts.

\begin{figure}[!htb]
  \centering
  \includegraphics[width=0.98\linewidth]{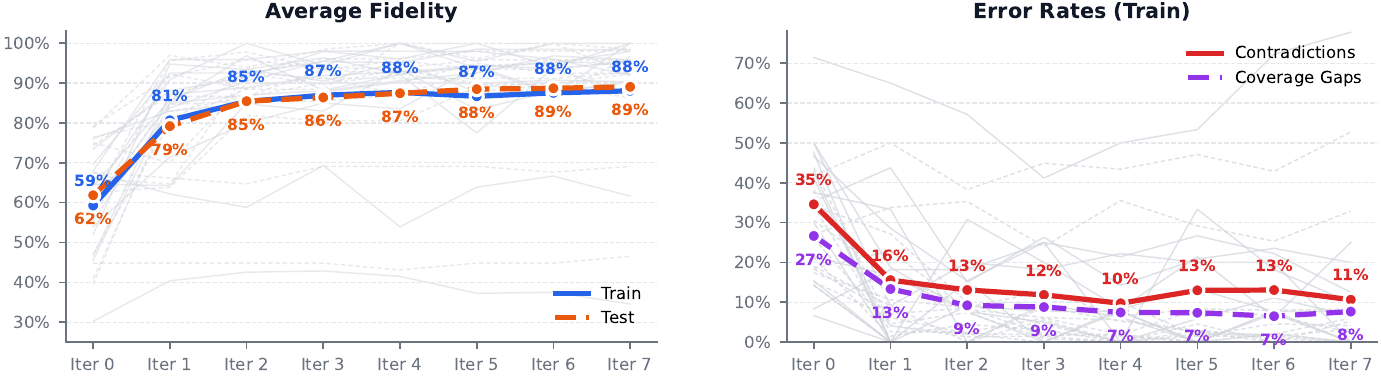}
  \caption{Program-mean fidelity (left) and curriculum error diagnostics (right). Solid and dashed curves show curriculum and held-out scores; grey curves show individual programs.}
  \label{fig:brief-dashboard}
\end{figure}

To understand how these gains arise, we apply the repair--regression analysis
from Section~\ref{sec:brief-formal} to successive versions of the specification.
Table~\ref{tab:brief-pi-rhat} counts probes that remain in agreement (held),
regress (regr), improve (impr), or remain in disagreement (stuck).
For $n$ probes, pooled fidelity is $F_k=(\mathrm{held}+\mathrm{regr})/n$
before revision and $F_{k+1}=(\mathrm{held}+\mathrm{impr})/n$ after revision.
The estimated rates are $\hat\pi_k=\mathrm{impr}/(\mathrm{impr}+\mathrm{stuck})$
and $\hat r_k=\mathrm{regr}/(\mathrm{held}+\mathrm{regr})$, giving
$\hat F^\dagger_k=\hat\pi_k/(\hat\pi_k+\hat r_k)$ under Eq.~\ref{eq:iclr-transition}.
The action count $|\mathcal A_k|$ records edits proposed from curriculum
failures (Eq.~\ref{eq:brief-revision}).

The measured rates reproduce the observed scores as a check on the
transition accounting (Appendix~\ref{app:brief-markov}).
Figure~\ref{fig:brief-pi-r} shows how they vary across revisions.
The core's early fixed-point estimate is close to later scores in this run.
In the final revision, 21 probes improve and 21 regress across all 15 programs,
leaving fidelity unchanged; the core alone has 19 repairs and 21 regressions.

\begin{table}[!htb]
  \centering
  \footnotesize
  \setlength{\tabcolsep}{2pt}
  \caption{Verdict transitions and pooled held-out fidelity before and after each revision.}
  \label{tab:brief-pi-rhat}
  \scriptsize
  \begin{tabularx}{\linewidth}{@{\extracolsep{\fill}}lcrrrrrcccrl@{}}
    \toprule
    Transition    & $F_k \to F_{k+1}$ & $|\mathcal{A}_k|$ & held & regr & impr & stuck & $\hat\pi_k$    & $\hat r_k$     & $\hat F^\dagger_k$ & $\mathrm{imp}{-}\mathrm{reg}$ & Progress?      \\
    \midrule
    $B_0 \to B_1$ & $0.590 \to 0.819$ & 330               & 439  & 24   & 204  & 118   & \textbf{0.634} & \textbf{0.052} & 0.924              & $+180$                        & \checkmark     \\
    $B_1 \to B_2$ & $0.819 \to 0.894$ & 117               & 622  & 21   & 80   & 62    & 0.563          & 0.033          & 0.945              & $+59$                         & \checkmark     \\
    $B_2 \to B_3$ & $0.894 \to 0.906$ & 85                & 687  & 15   & 24   & 59    & 0.289          & 0.021          & 0.931              & $+9$                          & \checkmark     \\
    $B_3 \to B_4$ & $0.906 \to 0.912$ & 70                & 698  & 13   & 18   & 56    & 0.243          & 0.018          & 0.930              & $+5$                          & \checkmark     \\
    $B_4 \to B_5$ & $0.912 \to 0.929$ & 56                & 699  & 17   & 30   & 39    & 0.435          & 0.024          & 0.948              & $+13$                         & \checkmark     \\
    $B_5 \to B_6$ & $0.929 \to 0.935$ & 63                & 716  & 13   & 18   & 38    & 0.321          & 0.018          & 0.947              & $+5$                          & \checkmark     \\
    $B_6 \to B_7$ & $0.935 \to 0.932$ & 61                & 713  & 21   & 19   & 32    & 0.373          & 0.029          & 0.929              & $-2$                          & at $F^\dagger$ \\
    \bottomrule
  \end{tabularx}
\end{table}

\begin{figure}[!htb]
  \centering
  \includegraphics[width=0.95\linewidth]{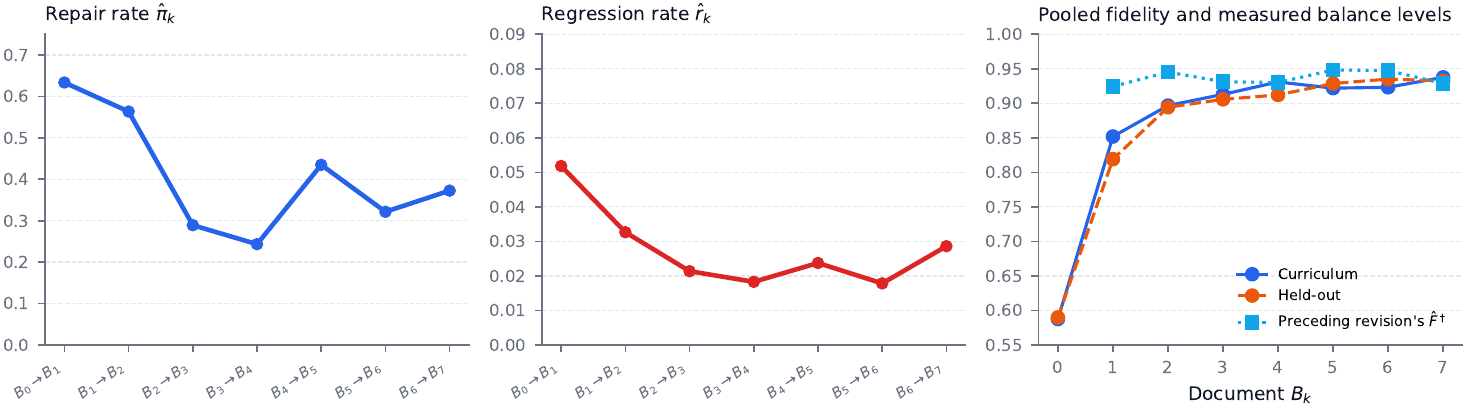}
  \caption{Repair and regression on the core cohort. Left: repair rate; middle: regression rate; right: pooled fidelity and balance levels measured for the preceding revision.}
  \label{fig:brief-pi-r}
\end{figure}

\subsection{Baselines and ablations}
\label{sec:baselines}
We compare three variants with the full loop on the same 15 CardDemo
programs, using a shared judge and 911 held-out probes
(Table~\ref{tab:baselines}). Each removes repeated revision, resampling, or
probe feedback. All perform seven revisions except the single-shot variant,
which performs one. Probe-based variants use the direct LLM channel.

\begin{table}[!htb]
  \centering
  \footnotesize
  \setlength{\tabcolsep}{8pt}
  \renewcommand{\arraystretch}{1.1}
  \caption{Revision variants on 15 CardDemo programs, with a shared judge and 911 held-out probes.}
  \label{tab:baselines}
  \begin{tabularx}{\linewidth}{@{}l>{\raggedright\arraybackslash}Xl@{}}
    \toprule
    \textbf{Variant} & \textbf{Isolates} & \textbf{Held-out test fidelity} \\
    \midrule
    Full loop (resampled probes) & None & $0.587 \to 0.885$ \\
    Single-shot revision & Iteration & $0.587 \to 0.784$ \\
    Fixed probes & Resampling & $0.587 \to 0.908$ (overfits) \\
    Free-form (Self-Refine style) & Probe feedback & $0.587 \to 0.817$ (plateaus) \\
    \bottomrule
  \end{tabularx}
\end{table}

Single-shot revision makes one pass using probe failures. Comparing it with
fixed probes examines further revision on the same initial curriculum.
Both iterative probe-guided variants achieve higher held-out fidelity.
The free-form baseline adapts Self-Refine \citep{madaan2023self} to
specifications. The same reviser reads the source and current specification
to improve their alignment, without probe questions, reference answers, or
verdicts. Held-out probes only score its output.
Probe feedback supplies explicit discrepancies to repair; free-form revision
relies on the reviser's own assessment of needed changes. Most of its gain
occurs in the first revision, with little subsequent improvement.
The resampled loop achieves 6.8 percentage points more held-out agreement
than free-form revision under the shared judge.
Section~\ref{sec:human-validation} separately assesses the accuracy of probe
answers and verdicts.
Fixed probes retain the initial curriculum throughout the loop.
This variant achieves the highest held-out fidelity, but scores better on
the reused curriculum at every revised checkpoint. Its mean curriculum--test
gap is 3.7 percentage points, compared with 0.1 for resampling.
These means average signed, pooled score differences over the seven revised
documents, using the same available programs for both scores at each checkpoint.
Missing intermediate program evaluations limit trajectory comparisons.
The ``overfits'' label refers to this gap; it does not establish fitting
to particular probe wordings.
To examine phrasing, we render 90 CFG guard facts from 19 programs with
both a deterministic template and the LLM informalizer.
The template inserts the condition and action verbatim; the informalizer
phrases the fact in business language. Opus~4.6 judges both against the
corresponding program's initial specification
(Table~\ref{tab:template-comparison}).

\begin{table}[!htb]
  \centering
  \footnotesize
  \setlength{\tabcolsep}{8pt}
  \renewcommand{\arraystretch}{1.1}
  \caption{Probe phrasing on the same 90 CFG facts, judged against the same initial specifications.}
  \label{tab:template-comparison}
  \begin{tabularx}{\linewidth}{@{}Xrrr@{}}
    \toprule
    \textbf{Probe phrasing} & \textbf{Fidelity} & \textbf{Coverage gap} & \textbf{Contradiction} \\
    \midrule
    Deterministic template & $22.2\%$ & $67.8\%$ & $10.0\%$ \\
    LLM informalizer & $\mathbf{44.4\%}$ & $41.1\%$ & $14.4\%$ \\
    \bottomrule
  \end{tabularx}
\end{table}

On the 34 facts with differing agreement decisions, the informalizer agrees
on 27 and the template on seven.
Gap verdicts fall from 61 to 37, while contradictions rise from nine to thirteen.
Template questions retain source identifiers that a business specification
can legitimately omit, making the underlying rule harder to recognize.
The renderings can ask different questions about the same fact,
and the informalizer can add or omit context. The higher agreement therefore
does not establish that every informalized question preserves the extracted
fact's meaning. Appendix~\ref{app:template-study} gives the protocol and examples.

\subsection{Graph-grounded probe mixtures}
\label{sec:brief-neuro-eval}
We next vary where the questions come from while retaining the judge and
reviser. Table~\ref{tab:brief-neurosymbolic} and
Figure~\ref{fig:trajectories-side-by-side} compare direct, graph-grounded,
and mixed generation using program-mean summaries. The graph mixtures share
an eleven-program evaluation cohort, and the direct reference uses the
thirteen-program core. Every mixture improves on its own held-out probe set,
but different questions
and cohorts prevent a controlled ranking of their absolute scores.
Appendix~\ref{app:brief-neurosymbolic} gives results for each mixture.
The cross-distribution study scores revised documents on questions from another
channel to examine which gains transfer
(Figure~\ref{fig:brief-cross-distribution}).

\begin{table}[!htb]
  \centering
  \footnotesize
  \setlength{\tabcolsep}{8pt}
  \caption{Program-mean fidelity by probe mixture. ``Best test'' reports the highest held-out score.}
  \label{tab:brief-neurosymbolic}
  \renewcommand{\arraystretch}{1.12}
  \begin{tabular}{@{}lccc@{}}
    \toprule
    Mixture & \shortstack{Baseline\\(train\,/\,test)} & \shortstack{$B_7$\\train} & \shortstack{Best\\test} \\
    \midrule
    pure LLM ($\alpha{=}1$, 13-prog core)                 & $0.61$\,/\,$0.63$         & $\mathbf{0.94}$ & $\mathbf{0.94}$ ($B_7$) \\
    pure CFG ($\alpha{=}0$, $\beta_{\mathrm{cfg}}{=}1$)   & $0.67$\,/\,$0.64$         & $0.84$          & $0.83$ ($B_5$)          \\
    half LLM ($\alpha{=}0.5$, $\beta_{\mathrm{cfg}}{=}1$) & $0.63$\,/\,$0.64$         & $0.79$          & $0.80$ ($B_5$)          \\
    balanced sym ($\alpha{=}0$, $\beta{=}(.34,.33,.33)$)  & $0.54$\,/\,$0.54$         & $0.83$          & $\mathbf{0.84}$ ($B_7$) \\
    \midrule
    $\Delta$ from baseline (balanced sym)                 &                           & $+29$pp         & $+30$pp                  \\
    \bottomrule
  \end{tabular}
\end{table}

\subsection{Generator, judge, and reviser substitutions}
\label{sec:brief-generator-sweep}
\label{sec:brief-judge-swap}
This study asks whether improvement depends on the model that writes the
curriculum. We replace the generator while retaining Opus as judge and reviser,
with variants spanning five model lineages
(Appendix~\ref{app:generator-sweep}, Table~\ref{tab:generator-sweep}).
Sonnet, Qwen, and GPT-OSS improve held-out agreement; DeepSeek and Gemma
show larger differences between curriculum and held-out scores
(Figure~\ref{fig:trajectories-side-by-side}, left).
The held-out set comes from the original generator. A difference can
therefore arise even if the replacement generator is stable across iterations,
because its questions cover different behaviour. The sweep does not isolate
overfitting or temporal drift as the cause.

\begin{figure}[!htb]
  \centering
  \begin{minipage}[t]{0.49\linewidth}
    \centering
    \includegraphics[width=\linewidth]{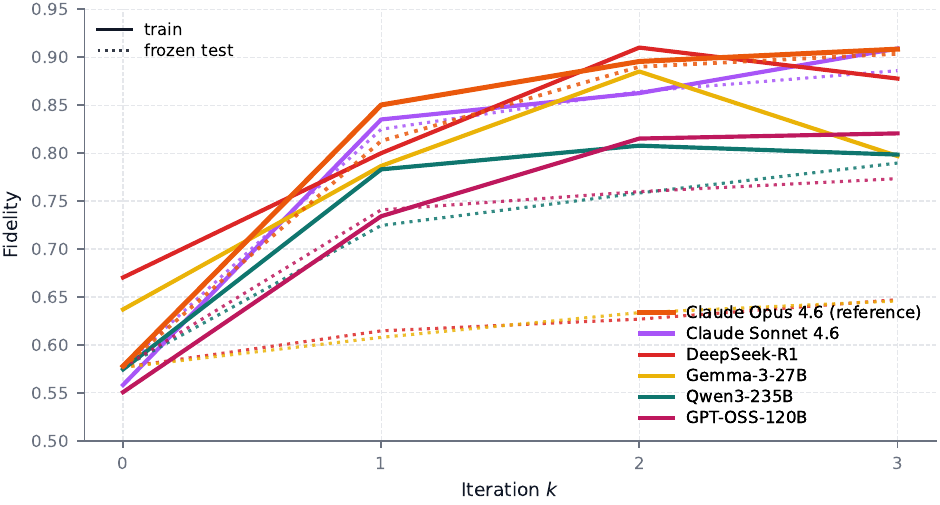}
    \subcaption{Generator-family sweep.}
    \label{fig:generator-sweep}
  \end{minipage}\hfill
  \begin{minipage}[t]{0.49\linewidth}
    \centering
    \includegraphics[width=\linewidth]{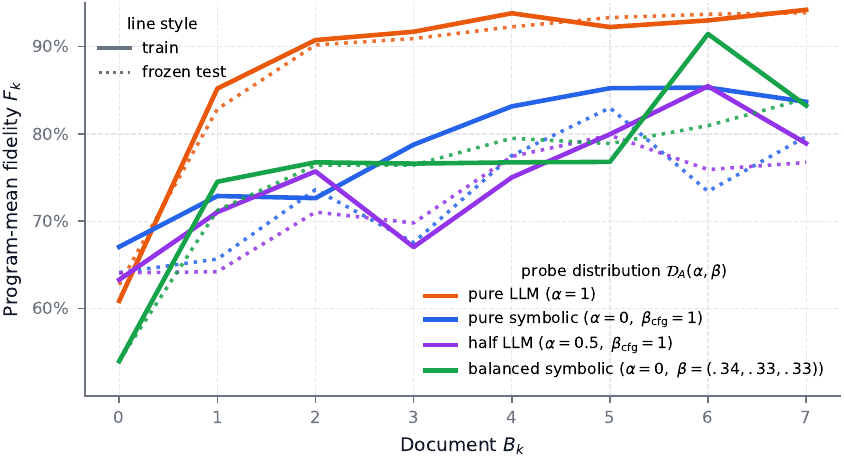}
    \subcaption{Neurosymbolic mixtures $\mathcal{D}_A(\alpha,\beta)$.}
    \label{fig:brief-neurosymbolic-trajectory}
  \end{minipage}
  \caption{Fidelity trajectories across generators (left) and probe mixtures (right). Solid and dotted curves show curriculum and held-out scores, respectively. The judge and reviser are fixed.}
  \label{fig:trajectories-side-by-side}
\end{figure}

To test sensitivity to the other model roles, we re-judge the same
specifications using Qwen and GPT-OSS and separately use these models for
one-step revision (Figure~\ref{fig:judge-swap}).
Changing the judging model changes both the specification answerer and
answer comparator. On the common 785-probe core, pooled gains through $B_3$
are 31.6, 25.5, and 27.3 points for Opus, Qwen, and GPT-OSS.
In separate runs with each model as reviser, fidelity after one revision
is 0.871, 0.850, and 0.834, respectively. All three revisers improve agreement.
We use human review to examine errors that agreement between models could miss.

\begin{figure}[!htb]
  \centering
  \begin{minipage}[c]{0.49\linewidth}
    \centering
    \includegraphics[width=\linewidth]{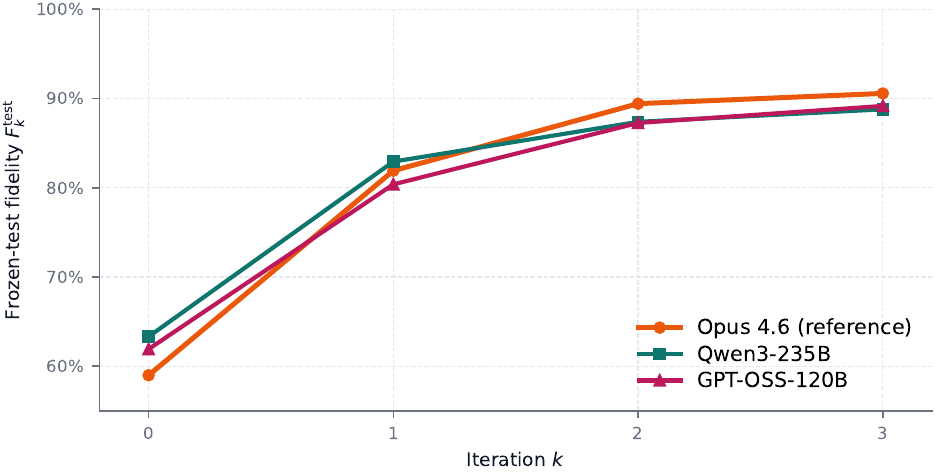}
  \end{minipage}\hfill
  \begin{minipage}[c]{0.50\linewidth}
    \scriptsize
    \setlength{\tabcolsep}{1.5pt}
    \renewcommand{\arraystretch}{1.15}
    \begin{tabularx}{\linewidth}{@{}lXXXXX@{}}
      \toprule
      Judge          & $F_0$   & $F_1$   & $F_2$   & $F_3$   & $\Delta F$ \\
      \midrule
      Opus 4.6 (ref) & $0.590$ & $0.819$ & $0.894$ & $0.906$ & $+31.6$    \\
      Qwen3-235B     & $0.633$ & $0.829$ & $0.874$ & $0.888$ & $+25.5$    \\
      GPT-OSS-120B   & $0.619$ & $0.804$ & $0.873$ & $0.892$ & $+27.3$    \\
      \bottomrule
    \end{tabularx}

    \vspace{4pt}
    \begin{tabularx}{\linewidth}{@{}lXXX@{}}
      \toprule
      Reviser ($B_0\!\to\!B_1$) & $F_0$   & $F_1$   & $\Delta F$    \\
      \midrule
      Opus 4.6 (ref)            & $0.590$ & $0.871$ & $+28.2$     \\
      Qwen3-235B                & $0.590$ & $0.850$ & $+26.0$     \\
      GPT-OSS-120B              & $0.590$ & $0.834$ & $+24.5$     \\
      \bottomrule
    \end{tabularx}
  \end{minipage}
  \caption{Judge and reviser substitutions on 785 held-out probes. Gains are in percentage points. Separate one-step reviser runs are judged by Opus.}
  \label{fig:judge-swap}
\end{figure}

\FloatBarrier
\subsection{Additional systems}
\label{sec:further-systems}

\begin{table}[!htb]
  \centering
  \footnotesize
  \setlength{\tabcolsep}{8pt}
  \renewcommand{\arraystretch}{1.1}

  \caption{Systems evaluated. Arrows show fidelity before and after revision; single scores are measurements without revision. Source sizes are approximate; System A includes copybooks.}
  \label{tab:systems}
  \begin{tabular}{@{}llrrrr@{}}
    \toprule
    \textbf{System} & \textbf{Language} & \textbf{Programs} & \textbf{kLoC} & \textbf{Probes} & \textbf{Fidelity} \\
    \midrule
    CardDemo & COBOL & 15 & 12 & 911 & $0.59 \to 0.89$ \\
    System A & COBOL & 352 & 480 & 23,526 & $0.86$ \\
    System B & PL/1 & 32 & 57 & 1,668 & $0.58 \to 0.86$ \\
    Redis docs & C & 6 modules & 6.7 & 600 & $0.58$ \\
    PetClinic & Java & 25 files & 1.8 & 203 & $0.74 \to 0.99$ \\
    \bottomrule
  \end{tabular}
\end{table}

The additional evaluations cover larger codebases, a second mainframe
language, and human-authored documentation (Table~\ref{tab:systems}).
On System A, a further abstraction step reduces fidelity from 0.857 to 0.426,
mainly through coverage gaps: the later document omits behaviour described in
the earlier one. On System B, one revision raises fidelity from 0.576 to 0.856
on 1,668 fixed probes; its separate measurement bank contains 1,756.

Redis's official command documentation yields 13 contradictions and 242 gaps
in 600 probes. Some gaps reflect the intended scope of documentation, so the
score should not be read as a count of documentation defects.
The additional public Java case, PetClinic, improves from 0.744 to 0.990 over
three revisions on 203 fixed questions.
The private systems provide evaluations on production codebases; their source is unavailable
for independent reproduction.

\FloatBarrier
\ifdim\dimexpr\pagegoal-\pagetotal\relax<22\baselineskip
  \newpage
\fi
\subsection{Human validation}
\label{sec:human-validation}

\begin{wraptable}{r}{0.56\textwidth}
  \centering
  \fontsize{8pt}{10pt}\selectfont
  \setlength{\tabcolsep}{6pt}
  \renewcommand{\arraystretch}{1.1}
  \captionsetup{skip=3pt}
  \caption{Human checks: source review of probes, verdicts, and edits; blind equivalence judgments; and comparison with an independent requirements audit.}
  \label{tab:human-checks}
  \begin{tabular}{@{}lcc@{}}
    \toprule
    \textbf{Expert source audit (CardDemo)} & \textbf{$n$} & \textbf{Correct} \\
    \midrule
    Probe reference answers & 120 & 119 \\
    Judge verdicts, initial specification & 120 & 93 \\
    Judge verdicts, final specification & 120 & 109 \\
    Applied edits & 30 & 25 \\
    \midrule
    \textbf{Blind equivalence check} & \textbf{$n$} & \textbf{Verdict confirmed} \\
    \midrule
    Pairs the comparator called different & 3 & 3 \\
    Pairs the comparator called equivalent & 50 & 48 \\
    \midrule
    \textbf{Requirements audit (System B)} & \textbf{$n$} & \textbf{Flagged by probes} \\
    \midrule
    Requirements flagged by auditors & 43 & 33 (77\%) \\
    Requirements not flagged by auditors & 829 & 374 (45\%) \\
    \bottomrule
  \end{tabular}
\end{wraptable}

Table~\ref{tab:human-checks} summarizes three human checks.
An expert checked probe reference answers, judge verdicts, and applied edits
against the \mbox{CardDemo} source, with access to AI-prepared assessments.
Five of the 30 edits introduced incorrect claims.
A second reviewer judged whether paired probes had the same meaning
without seeing the comparator's verdicts. Two pairs marked equivalent by
the comparator differed in scope.
Professional analysts reviewed System B's requirements independently of our
pipeline. Requirements the analysts flagged were more likely to be
cited by failing probes
(odds ratio 4.0, Fisher exact $p=0.0001$).
The audit was not exhaustive, and our comparison used incomplete document
excerpts. The 45\% rate for unflagged requirements is therefore not a
false-positive estimate.

Two further expert panels reviewed separate samples from later CardDemo runs.
Three specialists assessed 24 probes and classified eight as technical noise
by majority vote. A second panel of five specialists,
including those three, assessed 24 findings after model-based
screening for suspect code answers and removal of duplicate records.
The majority confirmed 15 findings, rejected seven, and tied on two.
Agreement on finding validity was low overall.
Reviewers unanimously confirmed six of seven findings about data and rule
content, but navigation and screen behaviour divided them.
These disagreements show why the intended scope of the modernized system
must be explicit when judging omissions.
Appendix~\ref{app:manual-eval} gives the protocols and agreement statistics,
including model--human comparisons (Table~\ref{tab:kappa}).

\section{Limitations}
\label{sec:brief-limitations}
The indirection through probes relies on LLMs to phrase questions and answers
and judge agreement. Errors in these steps can distort scores and misdirect
revisions, so the loop need not reach full measured fidelity. Probes may also
miss cross-file behaviour, and shared errors can persist through resampling.
Reviewers need the modernization scope to distinguish intentional interface
omissions from defects. High fidelity does not certify behavioural equivalence.
We have no independent assessment of complete revised documents. Evaluation alone used
434.8 million input tokens and 39.6 summed program-hours; revision calls are
additional (Appendix~\ref{app:brief-cost}).

\section{Conclusion}
\label{sec:brief-conclusion}
We contribute (i) fidelity probes that compare answers from natural-language
specifications with code-derived reference answers to questions about observable behaviour;
(ii) a revision loop using fresh direct or graph-derived probes and fixed
held-out evaluation; and (iii) a repair--regression analysis explaining how
new failures can offset repairs. For human audits, the probes present
questions in business language alongside answers from code and specifications,
giving domain experts specific behavioural claims and cited evidence to examine.

On 15 CardDemo programs, seven revisions raise held-out fidelity from 0.587
to 0.885 on 911 questions, compared with 0.817 for free-form revision and
0.908 for fixed probes. Additional studies cover industrial COBOL and PL/1
systems, PetClinic, and Redis documentation. Human reviews identify errors
in probes and revisions. Whether the revised specifications improve modernized
software remains to be evaluated.

\label{iclr-main-text-end}
\FloatBarrier
\bibliographystyle{iclr2027_conference}
\bibliography{refs}

\clearpage
\appendix

\FloatBarrier
\section{Proofs and formal statements}
\label{app:proofs}
The transition identity is an accounting result on a fixed probe set.
Convergence and concentration require additional assumptions, stated below.
In particular, a lower bound on improvement does not establish contraction
or a ceiling on attainable fidelity.

\subsection{One-step improvement on a fixed probe set}
\begin{assumption}[Conditional repair and regression bounds]
\label{asm:nonreg}\label{app:asm:nonreg}
Fix a document $B_k$ and a probe set $T$. For each probe that agrees on $B_k$,
the probability of agreement on $B_{k+1}$ is at least $1-r_k$.
For each failing probe, that probability is at least $\pi_k$.
Probabilities are conditional on $B_k,T$ and any fixed judging procedure.
\end{assumption}
\begin{theorem}[Expected one-step improvement]
\label{thm:monotone}\label{app:thm:monotone}
Under Assumption~\ref{asm:nonreg},
\begin{equation}
\mathbb E[F(B_{k+1};T)\mid B_k,T]
\ge (1-r_k)F(B_k;T)+\pi_k(1-F(B_k;T)).
\end{equation}
For $\pi_k+r_k>0$, expected progress on $T$ is non-negative if
$F(B_k;T)\le\pi_k/(\pi_k+r_k)$.
\end{theorem}
\begin{proof}
Partition $T$ into agreeing and failing probes on $B_k$.
Apply the respective probability bound to each indicator of agreement on
$B_{k+1}$, sum the bounds, and divide by $|T|$.
Independence between probes is unnecessary for this linearity argument.
The progress condition follows by subtracting $F(B_k;T)$.
\end{proof}

\begin{corollary}[Deviation lower bound]\label{app:cor:geometric}
Let $F_k^\dagger=\pi_k/(\pi_k+r_k)$ when the denominator is positive.
The preceding result gives
\begin{equation}
\mathbb E[F(B_{k+1};T)\mid B_k,T]-F_k^\dagger
\ge (1-\pi_k-r_k)(F(B_k;T)-F_k^\dagger).
\end{equation}
\end{corollary}
\begin{proof}
Subtract $F_k^\dagger$ and use
$\pi_k=(\pi_k+r_k)F_k^\dagger$.
\end{proof}
This is a lower bound, not a bound on the absolute deviation.
It does not prove convergence. It also concerns the same $T$ on both
versions, not a new curriculum sampled after revision.

\subsection{Exact transition recurrence}
On a fixed probe set, define $\pi_k$ and $r_k$ as the observed conditional transition
fractions. Counting the retained agreements and repaired failures gives
Eq.~\ref{eq:iclr-transition} exactly. If either starting class is empty,
its conditional rate is undefined, but its contribution to the count is zero.

\begin{theorem}[Convergence with constant rates]\label{app:thm:fixedpoint}
For the exact recurrence with constant $\pi,r$, $\pi+r>0$, and
$|1-\pi-r|<1$, the sequence converges to
$F^\dagger=\pi/(\pi+r)$.
\end{theorem}
\begin{proof}
Subtract the fixed-point equation $F^\dagger=(1-\pi-r)F^\dagger+\pi$.
This gives $F_{k+1}-F^\dagger=(1-\pi-r)(F_k-F^\dagger)$.
Iteration gives Theorem~\ref{thm:fixedpoint}; the multiplier's absolute
value is below one, so the deviation tends to zero.
\end{proof}
The same conclusion holds for an exact variable-rate recurrence if the rates
converge and $|1-\pi_k-r_k|\le\rho<1$ eventually, with a positive limiting
sum. To see this, subtract the limiting fixed point to obtain
$|e_{k+1}|\le\rho|e_k|+\eta_k$, where $\eta_k\to0$.
For any $\epsilon>0$, choose a late enough index that
$\eta_k\le(1-\rho)\epsilon$. Iteration then gives
$\limsup_k|e_k|\le\epsilon$, and letting $\epsilon\to0$ proves convergence.
The measured rates are not constant, so fitting this recurrence to an observed
transition is an explanation of that transition, not an independent prediction.

\subsection{Finite-sample concentration}
\label{rem:scope}
\begin{assumption}[Independent evaluation samples]\label{asm:iid}\label{app:asm:iid}
At a fixed checkpoint, $B_k$ is independent of the evaluation probes.
The test bank consists of $n_{\rm test}$ i.i.d. draws from a fixed distribution.
For a curriculum--test bound, the curriculum must also consist of
$n_{\rm train}$ independent draws from that distribution, independent of $B_k$.
The judging procedure is fixed, or its randomness is included in the independent
draws. Let $\mu_k$ be the resulting population agreement probability.
\end{assumption}
\begin{theorem}[Conditional sample-error bound]\label{thm:gap}\label{app:thm:gap}
Under Assumption~\ref{asm:iid}, for $\delta\in(0,1)$,
\begin{equation}
\Pr\!\left[|F_k^{\rm test}-\mu_k|\ge
\sqrt{\frac{\log(2/\delta)}{2n_{\rm test}}}\right]\le\delta.
\end{equation}
With probability at least $1-2\delta$,
\begin{equation}
|F_k^{\rm train}-F_k^{\rm test}|
\le\sqrt{\frac{\log(2/\delta)}{2n_{\rm train}}}
+\sqrt{\frac{\log(2/\delta)}{2n_{\rm test}}}.
\end{equation}
\end{theorem}
\begin{proof}
Apply Hoeffding's inequality to the bounded agreement indicators.
For the second statement apply it to each sample and combine the two events
by a union bound and the triangle inequality.
\end{proof}
The bounds describe sampling variation around the population agreement rate.
The accuracy of reference answers and judgments must be assessed separately.
The implemented graph quotas create dependent draws and do not meet the
i.i.d. premise.
An observed score difference above the bound does not identify which assumption
failed. Changing generators can change the question distribution.
Shared specifications can also couple program
results. Withholding test questions from the reviser is useful, but does not
justify selection based on test scores. Uniform claims over multiple checkpoints
need a simultaneous bound; individual Wilson intervals are not such a bound.

\subsection{A minimum-progress stopping rule}
\begin{proposition}[Finite stopping]\label{prop:termination}\label{app:prop:termination}
For any scores $F_k\in[0,1]$, stop at the first $k\ge1$ for which
$F_k-F_{k-1}<\delta$, where $\delta>0$.
The rule stops by $\lfloor(1-F_0)/\delta\rfloor+1$.
\end{proposition}
\begin{proof}
Every continued step adds at least $\delta$. After $m$ such steps,
$F_m\ge F_0+m\delta\le1$, so at most
$\lfloor(1-F_0)/\delta\rfloor$ steps can continue.
\end{proof}
This is a pathwise bound on the stopping rule. It does not guarantee document
correctness or monotonic expected scores. Our reported experiments instead use
a fixed revision budget.

\FloatBarrier
\section{Question overlap across revisions}
\label{app:brief-stationarity}

We examine similarity between successive question samples. Questions are embedded
with Amazon Titan Embed Text v2 (1,024 dimensions) and normalized. Each later question is compared
with initial questions from the same program using its maximum cosine similarity.
Scores at least 0.85 count as reuse, scores below 0.70 as novel, and the remainder
as related. These are heuristic cutoffs for a descriptive comparison, not a
calibrated test of semantic equivalence or total-variation distance.

\begin{table}[h]
  \centering
  \footnotesize
  \caption{Question overlap across 15 programs, relative to the initial questions. The last row reports initially failing probes revisited by later curricula, using the same near-match threshold of 0.85.}
  \label{tab:brief-overlap}
  \begin{tabular}{lccccccc}
    \toprule
                                     & Iter 1 & Iter 2 & Iter 3 & Iter 4 & Iter 5 & Iter 6 & Iter 7 \\
    \midrule
    Cohort mean similarity $\mu$     & 0.768  & 0.758  & 0.753  & 0.761  & 0.761  & 0.760  & 0.761  \\
    Cohort reuse rate ($\ge 0.85$)   & 0.376  & 0.368  & 0.350  & 0.366  & 0.365  & 0.354  & 0.365  \\
    Cohort novelty rate ($< 0.70$)   & 0.347  & 0.396  & 0.403  & 0.360  & 0.375  & 0.365  & 0.372  \\
    Initial failures revisited       & 0.325  & 0.332  & 0.317  & 0.322  & 0.315  & 0.323  & 0.336  \\
    \bottomrule
  \end{tabular}
\end{table}

\begin{figure}[h]
  \centering
  \includegraphics[width=0.95\linewidth]{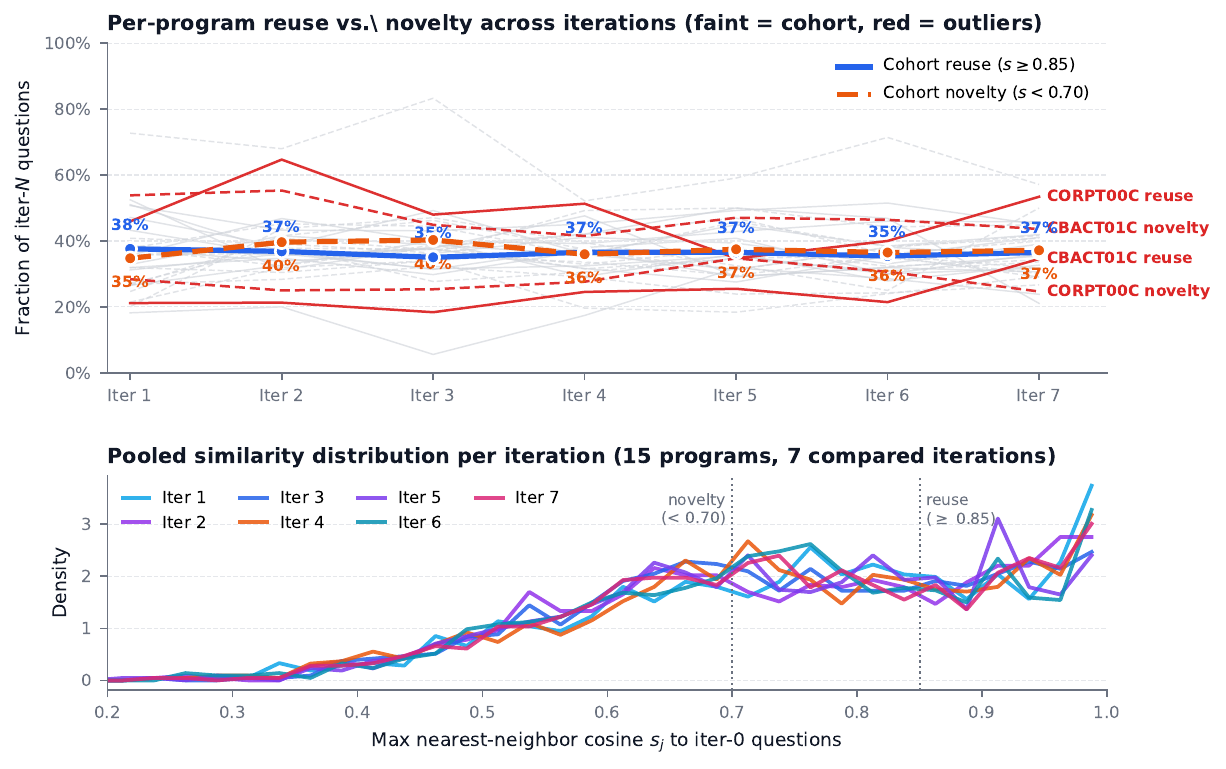}
  \caption{Question overlap. Top: reuse and novelty rates by program, with cohort means. Bottom: distributions of each question's maximum similarity to the initial questions.}
  \label{fig:brief-qa-overlap}
\end{figure}

Mean similarity varies by less than 0.02. Reuse and novelty rates also remain
similar, and later curricula revisit about one third of the initially failing probes.
These summaries do not rule out changes in which behaviours the questions cover.

\FloatBarrier
\section{Informalization stability (symbolic channels)}
\label{app:brief-stability}

The informalizer uses temperature zero, a fixed prompt, and no iteration-dependent
context. To measure variation between calls, we phrase each selected graph fact
twice and ask the comparator to judge question and answer equivalence.

\begin{table}[h]
  \centering
  \footnotesize
  \setlength{\tabcolsep}{5pt}
  \caption{Repeated informalization. ``Facts'' counts attempted fact pairs and ``Successful pairs'' counts pairs for which both calls succeeded. Of 283 successful pairs, the comparator accepts 280 as semantically equivalent. Wilson intervals describe comparator decisions, not a bound on semantic distribution drift.}
  \label{tab:brief-stability}
  \begin{tabular}{lccccc}
    \toprule
    Channel              & Facts $n$ & Successful pairs & Exact match & \multicolumn{2}{c}{Semantic-equiv.\ $95\%$ Wilson CI}                    \\
    \cmidrule(lr){5-6}
                         &           &                  & rate        & rate                                                  & interval         \\
    \midrule
    \textsc{cfg} (guard) & 100       & 99               & 45.5\%      & \textbf{99.0\%}                                       & $[0.945, 0.998]$ \\
    \textsc{dfg} (data)  & 100       & 100              & 12.0\%      & \textbf{100.0\%}                                      & $[0.963, 1.000]$ \\
    \textsc{cg} (flow)   & 84        & 84               & 16.7\%      & \textbf{97.6\%}                                       & $[0.917, 0.993]$ \\
    \midrule
    Pooled               & 284       & 283              & 25.1\%      & \textbf{98.9\%}                                       & $[0.969, 0.996]$ \\
    \bottomrule
  \end{tabular}
\end{table}

\textbf{Human re-check of the comparator.} A reviewer who could not see the
comparator's decisions assessed all three rejected pairs and 50 accepted pairs.
The reviewer upheld all three rejections and 48 acceptances, for 51/53 agreement
($\kappa=0.73$). The two disagreements concern narrower human interpretations.
Because the sample includes every rejection but only a subset of acceptances,
the pooled agreement is not an estimate for all probe pairs.
Appendix~\ref{app:manual-eval} gives the procedure.

\FloatBarrier
\section{Cross-family probe-generator sweep}
\label{app:generator-sweep}

The sweep changes the generator while holding the judging model and reviser at
Opus~4.6. It includes Sonnet~4.6, DeepSeek-R1, Gemma-3-27B, Qwen3-235B,
and GPT-OSS-120B, spanning five model lineages with the reference.
Each run produces four document versions after three revisions.
The test bank is retained from the original generator; curricula are generated
by the replacement model. The two banks can therefore have different question
distributions even if each generator is stable across calls.

\begin{table}[h]
  \centering
  \footnotesize
  \setlength{\tabcolsep}{5pt}
  \caption{Generator substitutions on the same twelve programs and 737 held-out probes. Scores pool agreements; $\Delta F$ is the held-out gain and $\max_k |\Delta_k|$ is the largest absolute curriculum--test difference over $B_0,\ldots,B_3$.}
  \label{tab:generator-sweep}
  \begin{tabular}{lccccc}
    \toprule
    Generator (role $Q$) & $F_0^{\text{test}}$ & $F_3^{\text{train}}$ & $F_3^{\text{test}}$ & $\Delta F$ (pp) & $\max_k |\Delta_k|$ \\
    \midrule
    Claude Opus 4.6 (reference) & $0.577$ & $0.908$ & $0.904$ & $+32.7$ & $0.038$ \\
    Claude Sonnet 4.6           & $0.577$ & $0.909$ & $0.886$ & $+30.9$ & $0.023$ \\
    Qwen3-235B                  & $0.577$ & $0.798$ & $0.790$ & $+21.3$ & $0.059$ \\
    GPT-OSS-120B                & $0.577$ & $0.821$ & $0.773$ & $+19.7$ & $0.055$ \\
    DeepSeek-R1                 & $0.577$ & $0.878$ & $0.647$ & $+7.1$  & $0.283$ \\
    Gemma-3-27B                 & $0.577$ & $0.797$ & $0.646$ & $+6.9$  & $0.252$ \\
    \bottomrule
  \end{tabular}
\end{table}

The table and Figure~\ref{fig:generator-sweep} restrict every condition to
programs with valid results in both samples at all four checkpoints.
This intersection excludes \textsc{COSGN00C} in addition to the two programs
excluded from the thirteen-program core. All conditions use the same initial
documents and held-out questions; their curriculum questions differ.

\textbf{Interpretation.} Several generators support continued improvement, while
DeepSeek and Gemma show larger curriculum--test discrepancies. A discrepancy
can reflect differences in question coverage, noisy judgments, or overfitting.
The experiment does not isolate those causes. The constant-rate analysis
explains a trajectory when its assumptions hold; the concentration bound
neither certifies convergence nor identifies temporal drift from a gap alone.

\FloatBarrier
\section{The revision loop}
\label{app:brief-loop-diagram}

\begin{figure}[h]
  \centering
  \includegraphics[width=0.98\linewidth]{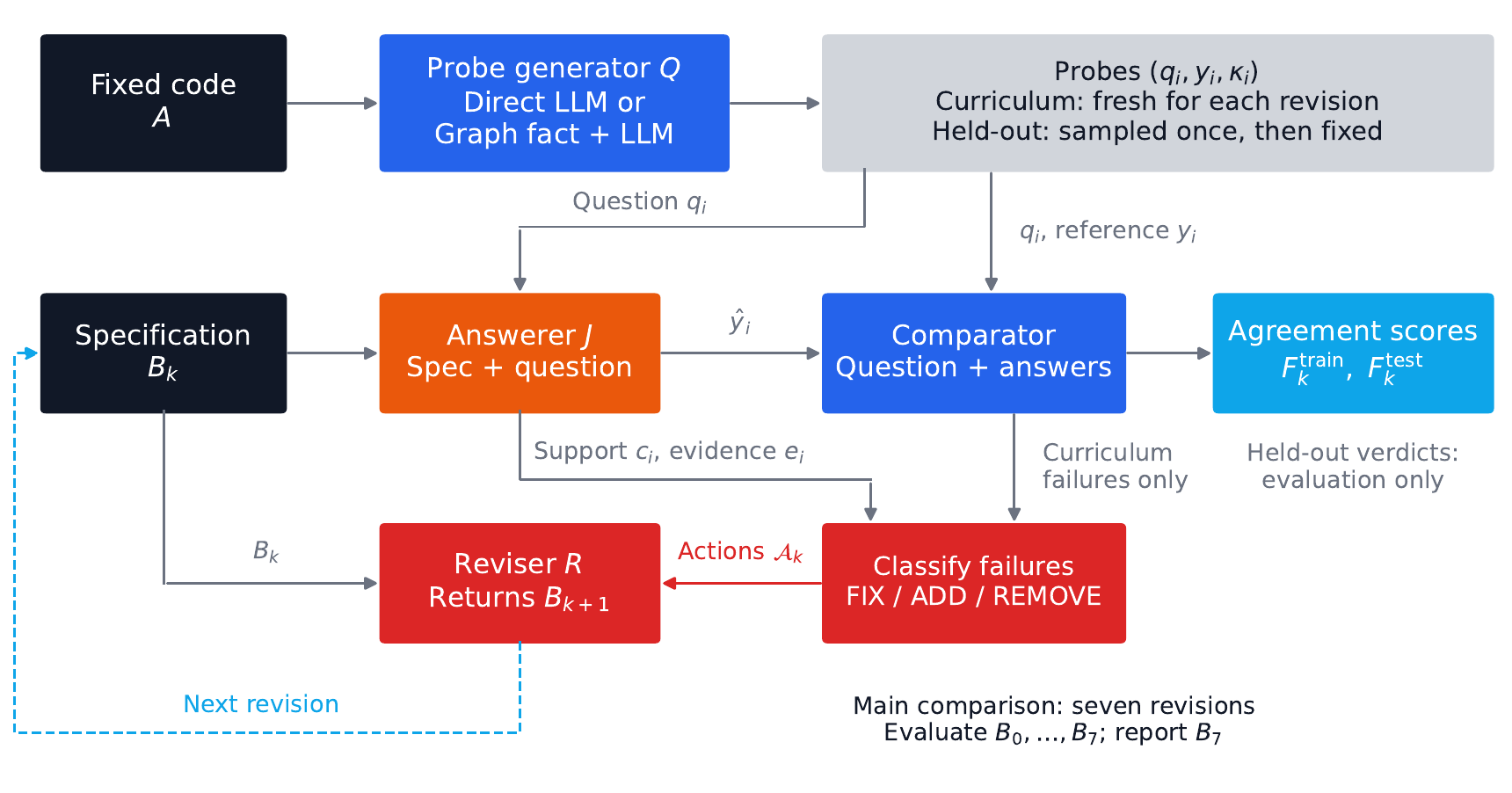}
  \caption{The revision loop. Direct generation and graph analysis provide alternative routes from the fixed code $A$ to probes. The answerer receives only the question and specification $B_k$; the comparator receives the question and both answers. Agreement determines fidelity, while rejected pairs and the support flag determine contradictions and gaps. Only curriculum failures produce \FIX{} / \ADD{} / \REMOVE{} actions for the reviser, which returns the complete document $B_{k+1}$. Held-out probes remain fixed and their verdicts are used only for evaluation. The loop returns $B_i$ after $i$ revisions.}
  \label{fig:brief-loop}
\end{figure}

\FloatBarrier
\section{Worked revisions: message details and a harmful edit}
\label{app:brief-repair-examples-symbolic}

The following examples trace three disagree$\to$agree transitions between
iterations $0$ and $1$ of the balanced-sym-v1 run
($\alpha{=}0$, Section~\ref{sec:brief-neuro-eval}), one from each graph channel.
Each example includes the probe and its reference answer, the relevant code,
the judge's answers before and after revision, and the intervening
\ADD{} or \FIX{} action. The first two add interface details whose relevance
depends on the modernization scope. The third changes a correct requirement
to match an incorrect reference answer.

\textbf{CFG (guard), \ADD: invalid ZIP/state message in \texttt{COACTUPC}.}

\noindent\textbf{Probe.}\; ``What message is displayed if the zip code does not match the selected state during account update?'' (channel \textsc{cfg}, category \textsc{guard}.)

\noindent\textbf{Code-derived reference answer.}\; \emph{If no prior error message has already been set, the system displays `Invalid zip code for state'.}
\begin{lstlisting}[basicstyle=\ttfamily\footnotesize,frame=single,xleftmargin=0pt]
* COACTUPC.cbl, paragraph 1280-EDIT-US-STATE-ZIP-CD:
IF VALID-US-STATE-ZIP-CD2-COMBO
   CONTINUE
ELSE
   SET INPUT-ERROR         TO TRUE
   SET FLG-STATE-NOT-OK    TO TRUE
   SET FLG-ZIPCODE-NOT-OK  TO TRUE
   IF WS-RETURN-MSG-OFF
      STRING 'Invalid zip code for state'
             DELIMITED BY SIZE
        INTO WS-RETURN-MSG
      END-STRING
   END-IF
END-IF
\end{lstlisting}

\noindent\textbf{Judge on $B_0$ (baseline spec).}\; ``The specification does not provide the exact text of the error message displayed when the ZIP code does not match the selected state. It only states that the system shall display an error message when the ZIP code prefix is not valid for the selected state.'' Verdict: \textbf{coverage-gap}.

\noindent\textbf{Action produced by the loop.}\; \ADD{} a requirement pinning the literal message text, anchored near the ZIP/state validation rule: \emph{``When the account-update ZIP/state combination is invalid, the system shall display the message `Invalid zip code for state' in the error-message area, provided no prior validation message has been set.''}

\noindent\textbf{Judge on $B_1$ (after one iteration).}\; ``The message `Invalid zip code for state' is displayed.'' Verdict: \textbf{agree}.

The baseline already describes ZIP/state validation; the edit adds the exact
message. The comparator's prompt permits agreement between an exact message
and a general error statement (Appendix~\ref{app:prompt:comparator}), so this
recorded gap also illustrates variability in how that policy is applied.

\textbf{DFG (computation), \ADD: unavailable-option message in \texttt{COADM01C}.}

\noindent\textbf{Probe.}\; ``What message is displayed to the user when they select an admin menu option that is not available?'' (channel \textsc{dfg}, category \textsc{computation}.)

\noindent\textbf{Code-derived reference answer.}\; \emph{The system composes the message `This option ' $+$ `is not installed ...' and displays it in the admin-menu message area.}
\begin{lstlisting}[basicstyle=\ttfamily\footnotesize,frame=single,xleftmargin=0pt]
* COADM01C.cbl, PROCESS-ENTER-KEY (unavailable branch):
MOVE SPACES   TO WS-MESSAGE
MOVE DFHGREEN TO ERRMSGC OF COADM1AO
STRING 'This option '          DELIMITED BY SIZE
       'is not installed ...'  DELIMITED BY SIZE
  INTO WS-MESSAGE
PERFORM SEND-MENU-SCREEN
\end{lstlisting}

\noindent\textbf{Judge on $B_0$ (baseline spec).}\; ``The specification does not state the exact message text displayed when an unavailable administration menu option is selected.'' Verdict: \textbf{coverage-gap}.

\noindent\textbf{Action produced by the loop.}\; \ADD{} a requirement anchored in the admin-menu flow: \emph{``When the user selects an administration menu option that is not installed, the system shall display the message `This option is not installed \ldots' in the message area (rendered in the informational-notice colour) and redisplay the administration menu screen.''}

\noindent\textbf{Judge on $B_1$ (after one iteration).}\; ``The system displays the message `This option is not installed \ldots' in the message area with green visual emphasis to indicate it is an informational notice rather than a critical error, and re-presents the administration menu screen.'' Verdict: \textbf{agree}.

The code sets the colour to green. Interpreting that colour as an
informational notice is an additional explanation supplied in the revision
and answer, rather than a business rule established by the shown statements.

\textbf{CallGraph (dependency), harmful \FIX: admin-menu exit target in \texttt{COADM01C}.}

\noindent\textbf{Probe.}\; ``Where is the user taken when they request to exit from the admin menu screen?'' (channel \textsc{callgraph}, category \textsc{dependency}.)

\noindent\textbf{Extracted reference answer.}\; \emph{Return to the originating screen stored in session context, defaulting to sign-on if the value is empty.}
This answer generalizes the shared return paragraph without accounting for
the assignment made on the exit path. The PF3 branch shown below occurs
in the main dispatch; the recorded listing mislabels its enclosing paragraph
as \texttt{PROCESS-ENTER-KEY}.
\begin{lstlisting}[basicstyle=\ttfamily\footnotesize,frame=single,xleftmargin=0pt]
* COADM01C.cbl, PROCESS-ENTER-KEY (PF3 branch):
WHEN DFHPF3
    MOVE 'COSGN00C' TO CDEMO-TO-PROGRAM
    PERFORM RETURN-TO-SIGNON-SCREEN

* COADM01C.cbl, RETURN-TO-SIGNON-SCREEN:
IF CDEMO-TO-PROGRAM = LOW-VALUES OR SPACES
    MOVE 'COSGN00C' TO CDEMO-TO-PROGRAM
END-IF
XCTL PROGRAM(CDEMO-TO-PROGRAM)
\end{lstlisting}

\noindent\textbf{Judge on $B_0$ (baseline spec).}\; ``The user is navigated to the sign-on function (sign-on screen).'' Verdict: \textbf{contradict} against the extracted reference answer.

\noindent\textbf{Action produced by the loop.}\; \FIX{} \texttt{REQ-EXIT-003}: \emph{``When the user requests to exit the administration menu, the system shall transfer control to the originating screen recorded in the session context; if no originating screen is recorded, the system shall default to the sign-on screen.''}

\noindent\textbf{Judge on $B_1$ (after one iteration).}\; ``The user is navigated to the originating screen stored in the session context. If no originating screen is specified (i.e., the value is empty), the system defaults to navigating to the sign-on function.'' Verdict: \textbf{agree}.

On the exit path, the code sets the destination to the sign-on program
before entering the return paragraph. The original requirement therefore
describes that path correctly. Revision changes it to match the more general
reference answer, and the comparator changes its verdict from disagreement
to agreement. In this case, measured improvement accompanies an incorrect
requirement change.

All three verdict changes contribute to the improvement count
$\mathrm{imp}_k$, but agreement alone does not establish that an edit is
appropriate. The first two concern interface details whose relevance depends
on the intended modernized system; the third is a source-inconsistent edit.
Appendix~\ref{app:manual-eval} reports the human checks of probes and revisions.

\textbf{Verified normalization rule from the source audit.}
A separate direct-probe example, audit item B-011, asks: ``How is the card
name stored when reading from the database?'' The initial specification's
answer describes the embossed-name field but omits normalization. In
\texttt{COCRDUPC}, paragraph \texttt{9000-READ-DATA} converts that name to
uppercase before saving its original value. Paragraph
\texttt{9300-CHECK-CHANGE-IN-REC} applies the same conversion before comparing
it with the saved value to detect concurrent changes. The revision adds
requirement \texttt{REQ-UPD-021}, stating that normalization must precede
saving the original value for these comparisons. The source audit judged
this edit faithful and found no incorrect claim introduced. This example
illustrates a missing data-handling rule recovered by revision, in contrast
to the source-inconsistent exit change above.

\FloatBarrier
\section{Supporting results for the CardDemo revision trajectory}
\label{app:brief-markov}

The starting CardDemo documents form a curated set of functional requirements.
The mapping assigns \textsc{CBACT01C} and \textsc{CORPT00C} to the reporting
document, \textsc{CBTRN02C} and \textsc{COTRN01C} to transaction processing,
and \textsc{COCRDSLC} and \textsc{COCRDUPC} to card management.
\textsc{COSGN00C} uses six documents covering authentication, user administration,
bill payment, pending authorization, reporting, and transactions.
The remaining programs each have one mapped document. Across the fifteen
programs, eleven distinct requirements documents and a shared data-model
context are used. Shared documents mean that program evaluations need not
be independent.

This appendix gives the program-mean scores for the 13-program core cohort
(Table~\ref{tab:brief-headline}), the per-program fidelity heatmap
(Figure~\ref{fig:brief-heatmap}), and the balance identity for the final transition.
The per-transition counts and pooled scores are reported in
Table~\ref{tab:brief-pi-rhat}; Table~\ref{tab:brief-pred-vs-obs} checks their
reconstruction using the transition identity, and Figure~\ref{fig:brief-pi-r}
shows how the measured rates vary.

\begin{table}[!htb]
  \centering
  \footnotesize
  \setlength{\tabcolsep}{6pt}
  \caption{Program-mean fidelity trajectory on the 13-program core cohort. Column $k$ describes document $B_k$, with $B_0$ the initial specification. These scores average programs equally, whereas Table~\ref{tab:brief-pi-rhat} pools the cohort's 785 held-out probes.}
  \label{tab:brief-headline}
  \begin{tabular}{lcccccccc}
    \toprule
                       & $k{=}0$ & $k{=}1$ & $k{=}2$ & $k{=}3$ & $k{=}4$ & $k{=}5$ & $k{=}6$ & $k{=}7$        \\
    \midrule
    $F_k^\text{train}$ & 0.609   & 0.852   & 0.908   & 0.917   & 0.938   & 0.923   & 0.930   & \textbf{0.942} \\
    $F_k^\text{test}$  & 0.627   & 0.828   & 0.902   & 0.909   & 0.923   & 0.933   & 0.937   & \textbf{0.939} \\
    $|\Delta_k|$       & 0.019   & 0.024   & 0.006   & 0.008   & 0.016   & 0.011   & 0.007   & 0.003          \\
    \bottomrule
  \end{tabular}
\end{table}

\begin{table}[!htb]
  \centering
  \footnotesize
  \setlength{\tabcolsep}{5pt}
  \caption{Transition reconstruction on the 13-program core, using measured repair and regression rates in Eq.~\ref{eq:iclr-transition} and observed agreements over 785 probes.}
  \label{tab:brief-pred-vs-obs}
  \begin{tabular}{lcccccccc}
    \toprule
    $k$                                                                & 0     & 1     & 2     & 3     & 4     & 5     & 6     & 7     \\
    \midrule
    Reconstructed $F_k^\text{test}$ (Eq.~\ref{eq:iclr-transition}) & --    & 0.819 & 0.894 & 0.906 & 0.912 & 0.929 & 0.935 & 0.932 \\
    Observed $F_k^\text{test}$ (pooled, core)                          & 0.590 & 0.819 & 0.894 & 0.906 & 0.912 & 0.929 & 0.935 & 0.932 \\
    \bottomrule
  \end{tabular}
\end{table}

\begin{figure}[h]
  \centering
  \includegraphics[width=0.95\linewidth]{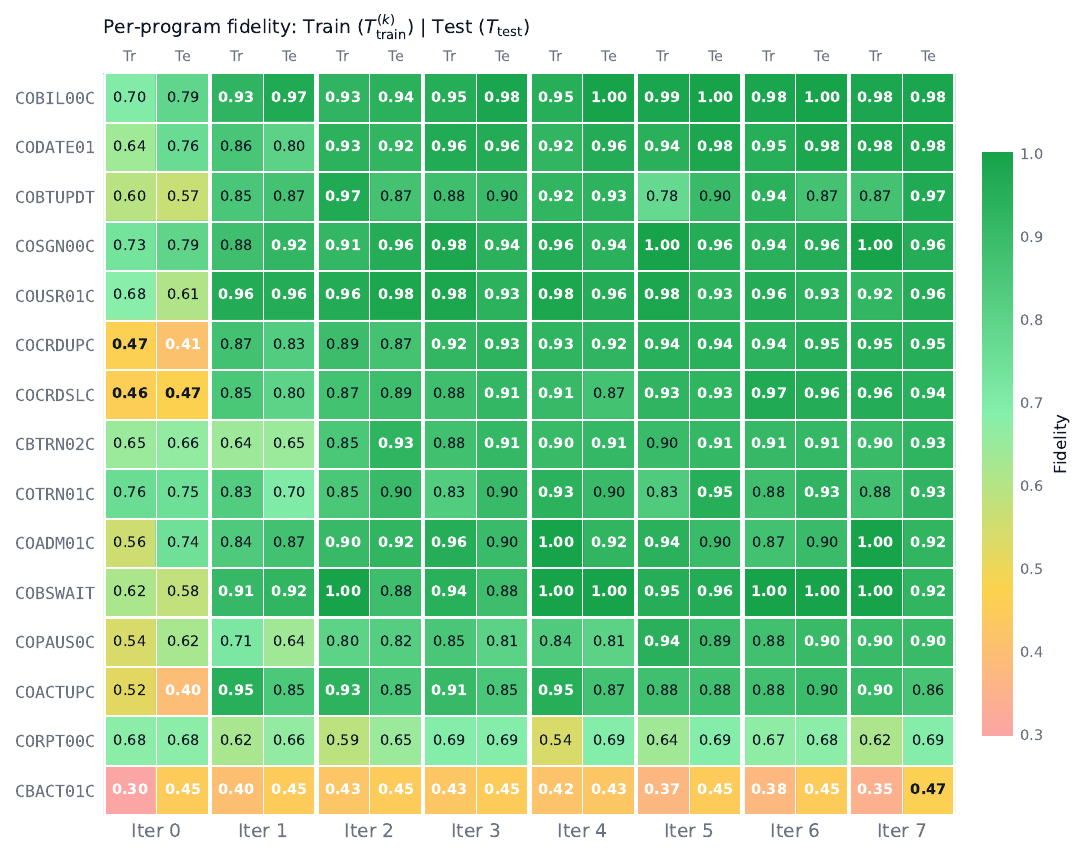}
  \caption{Per-program curriculum and held-out scores, sorted by final held-out fidelity. The difficult programs \textsc{CORPT00C} and \textsc{CBACT01C} depend on cross-program context that the direct generator reads only partially. Similar scores on two samples do not by themselves establish complete behavioural coverage.}
  \label{fig:brief-heatmap}
\end{figure}

\textbf{Stationary-state balance identity.} The two-state Markov chain (agree $\leftrightarrow$ disagree, rates $\hat\pi_k, \hat r_k$) has stationary distribution $\Pr[\text{agree}] = F^\dagger_k = \hat\pi_k/(\hat\pi_k+\hat r_k)$, at which flows in and out of the agree state balance:
\begin{equation}
  \mathbb{E}\!\left[\tfrac{\mathrm{imp}_k}{n}\right] = (1 - F_k)\,\hat\pi_k, \qquad \mathbb{E}\!\left[\tfrac{\mathrm{reg}_k}{n}\right] = F_k\,\hat r_k, \qquad \text{equal when } F_k = F^\dagger_k.
\end{equation}
For the final core transition $B_6\to B_7$, the measured rates are
$\hat\pi_6=19/51$ and $\hat r_6=21/734$, and $F_6=734/785$.
Substitution gives 19 repairs and 21 regressions, reproducing the observed
counts in Table~\ref{tab:brief-pi-rhat}. This is an accounting check using
rates from that same transition, not an independent forecast.

\FloatBarrier
\section{Cost and timing}
\label{app:brief-cost}

\begin{figure}[h]
  \centering
  \includegraphics[width=0.95\linewidth]{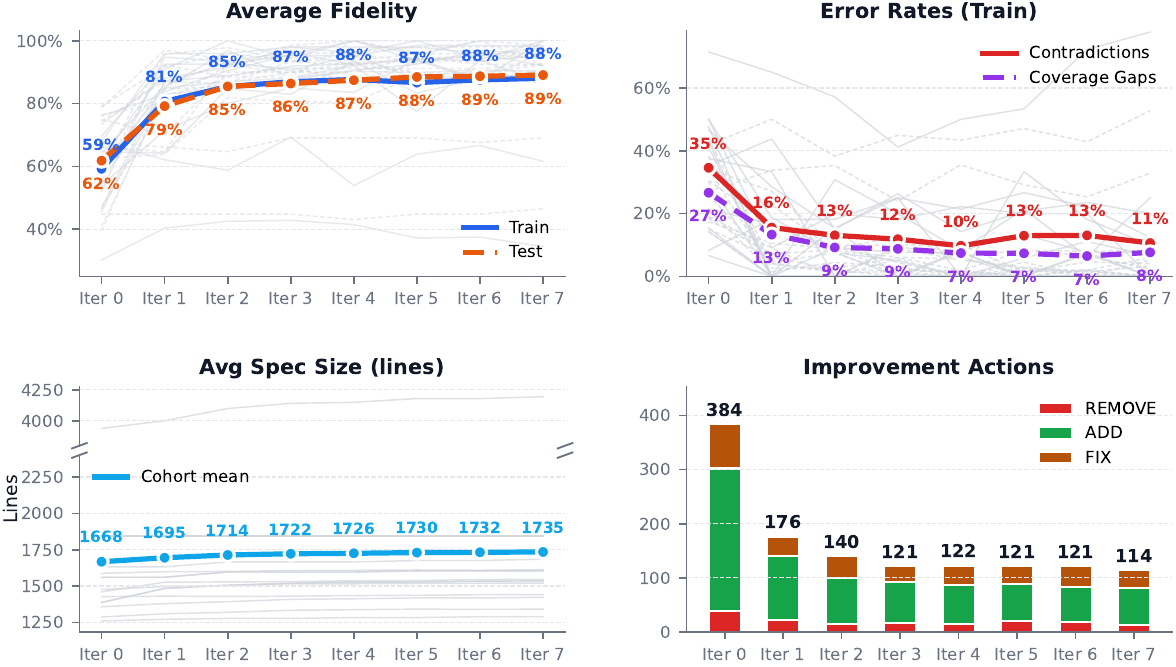}
  \caption{CardDemo dashboard. Top: program-mean fidelity and stored error diagnostics, as in Figure~\ref{fig:brief-dashboard}. Bottom: specification size in lines and proposed actions by type. All fifteen programs are included.}
  \label{fig:brief-dashboard-quad}
\end{figure}

Table~\ref{tab:brief-runcost} reports recorded evaluation usage;
Table~\ref{tab:brief-runcost-symbolic} compares the graph runs.
The latter evaluate fewer questions. Their lower totals therefore combine
differences in question volume and generation route.
Figure~\ref{fig:brief-timing} shows the relation between program size,
evaluation time, and token use.

\begin{table}[h]
  \centering
  \footnotesize
  \caption{Recorded evaluation cost for the direct-LLM CardDemo run. Hours are summed program runtimes, not elapsed time for concurrent execution. Revision calls are additional. Rows combine curriculum evaluation and the following held-out re-evaluation where recorded; iteration 7 has curriculum records only. The initial held-out evaluation is listed separately.}
  \label{tab:brief-runcost}
  \begin{tabular}{lrrrrr}
    \toprule
    Iteration & Hours          & Input           & Output        & Total           & Calls             \\
              & (h)            & (M tok)         & (M tok)       & (M tok)         &                   \\
    \midrule
    Initial test & 2.79 & 25.07 & 0.43 & 25.50 & 1{,}684 \\
    0 & 5.01 & 49.71 & 0.64 & 50.36 & 3{,}316 \\
    1 & 4.79 & 52.26 & 0.63 & 52.90 & 3{,}503 \\
    2 & 5.04 & 55.31 & 0.67 & 55.98 & 3{,}659 \\
    3 & 5.17 & 55.74 & 0.67 & 56.40 & 3{,}651 \\
    4 & 4.86 & 56.88 & 0.69 & 57.56 & 3{,}701 \\
    5 & 4.51 & 55.58 & 0.66 & 56.24 & 3{,}595 \\
    6 & 4.77 & 55.71 & 0.67 & 56.38 & 3{,}598 \\
    7 & 2.71 & 28.56 & 0.42 & 28.99 & 1{,}833 \\
    \midrule
    Total & \textbf{39.64} & \textbf{434.81} & \textbf{5.49} & \textbf{440.30} & \textbf{28{,}540} \\
    \bottomrule
  \end{tabular}
\end{table}

\textbf{Revision budget.} Most of the agreement gain occurs in the first few
revisions. The cost records sum to 434.8 million input tokens, 5.49 million
output tokens, and 28,540 evaluation calls. They exclude revision calls and
do not establish a minimum budget for a deployable specification.

\begin{table}[h]
  \centering
  \footnotesize
  \setlength{\tabcolsep}{4.5pt}
  \caption{Recorded evaluation cost for graph mixtures and the direct reference. Hours sum program runtimes. Graph runs use fewer questions, and their costs include LLM informalization and judging; revision calls are additional. These are different workloads, not a matched-cost comparison.}
  \label{tab:brief-runcost-symbolic}
  \begin{tabular}{lrrrrr}
    \toprule
    Mixture                                             & Hours & Input   & Output  & Total   & Calls    \\
                                                        & (h)   & (M tok) & (M tok) & (M tok) &          \\
    \midrule
    pure CFG ($\alpha{=}0, \beta_{\mathrm{cfg}}{=}1$)   & 2.9   & 35.6    & 0.40    & 36.0    & 2{,}482  \\
    half LLM ($\alpha{=}0.5, \beta_{\mathrm{cfg}}{=}1$) & 2.5   & 29.9    & 0.34    & 30.2    & 2{,}083  \\
    balanced sym ($\alpha{=}0, \beta{=}(.34,.33,.33)$)  & 5.3   & 67.6    & 0.74    & 68.3    & 4{,}606  \\
    \midrule
    pure LLM ($\alpha{=}1$, reference)                  & 39.6  & 434.8   & 5.49    & 440.3   & 28{,}540 \\
    \bottomrule
  \end{tabular}
\end{table}

\begin{figure}[h]
  \centering
  \begin{minipage}{0.49\linewidth}
    \centering
    \includegraphics[width=\linewidth]{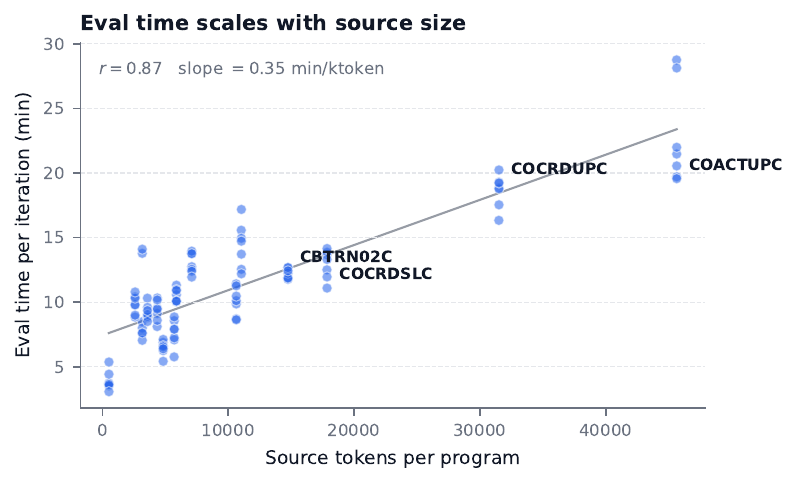}
  \end{minipage}\hfill
  \begin{minipage}{0.49\linewidth}
    \centering
    \includegraphics[width=\linewidth]{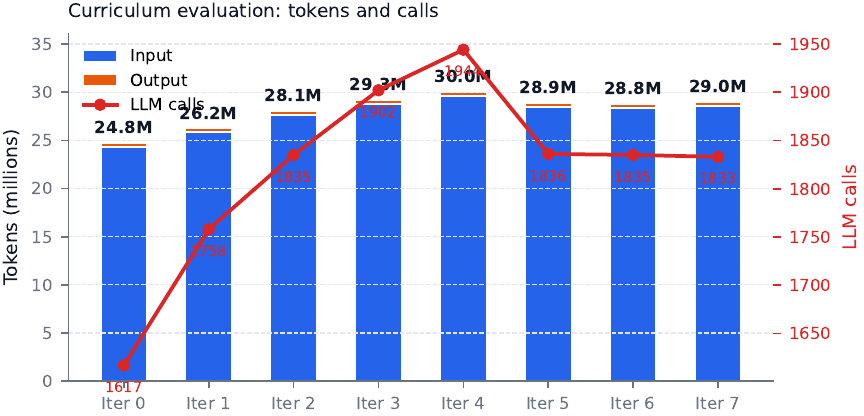}
  \end{minipage}
  \caption{Curriculum evaluation cost. Left: program runtime versus source tokens, with a linear fit. Right: input and output tokens (bars) and call counts (red line). These plots omit held-out evaluation, which is included in Table~\ref{tab:brief-runcost}.}
  \label{fig:brief-timing}
\end{figure}

\FloatBarrier
\section{Spec complexity and verdict decomposition}
\label{app:brief-complexity}

The stored history contains separately computed fidelity and error diagnostics.
For example, the initial curriculum program means are 0.593, 0.346, and 0.266,
which sum to more than one. Figure~\ref{fig:brief-verdict} retains that diagnostic
view; it is not the disjoint decomposition defined in the method.
Applying the agreement-first rule to the 911 held-out probes instead gives
$(F,C,G)=(535,120,256)/911$ initially and $(806,46,59)/911$ finally.
Figures~\ref{fig:brief-spec-complexity} and~\ref{fig:brief-complexity-correlations}
show document-size measures and their association with fidelity.

\begin{figure}[h]
  \centering
  \begin{minipage}{0.49\linewidth}
    \centering
    \includegraphics[width=\linewidth]{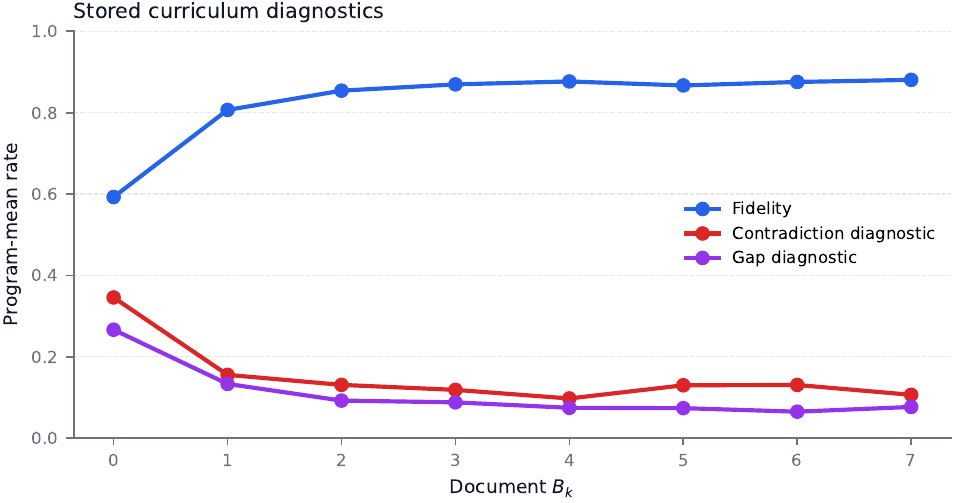}
  \end{minipage}\hfill
  \begin{minipage}{0.49\linewidth}
    \centering
    \includegraphics[width=\linewidth]{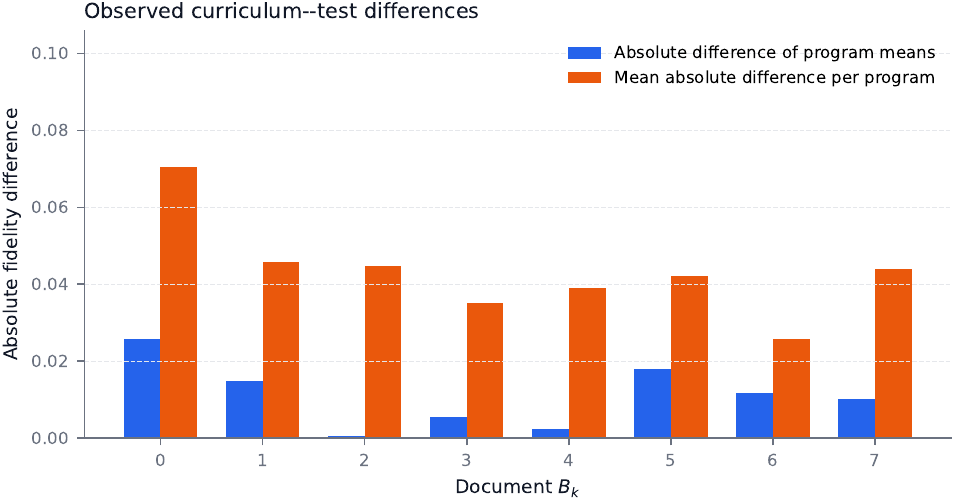}
  \end{minipage}
  \caption{Stored diagnostic rates (left) and curriculum--test score differences (right). The three fields in the left plot can overlap and must not be read as a probability partition. The disjoint counts are reported in Section~\ref{sec:brief-pure-llm}. The right plot compares the absolute difference of program means with the mean absolute difference within programs, across all fifteen programs.}
  \label{fig:brief-verdict}
\end{figure}

\begin{figure}[h]
  \centering
  \includegraphics[width=0.95\linewidth]{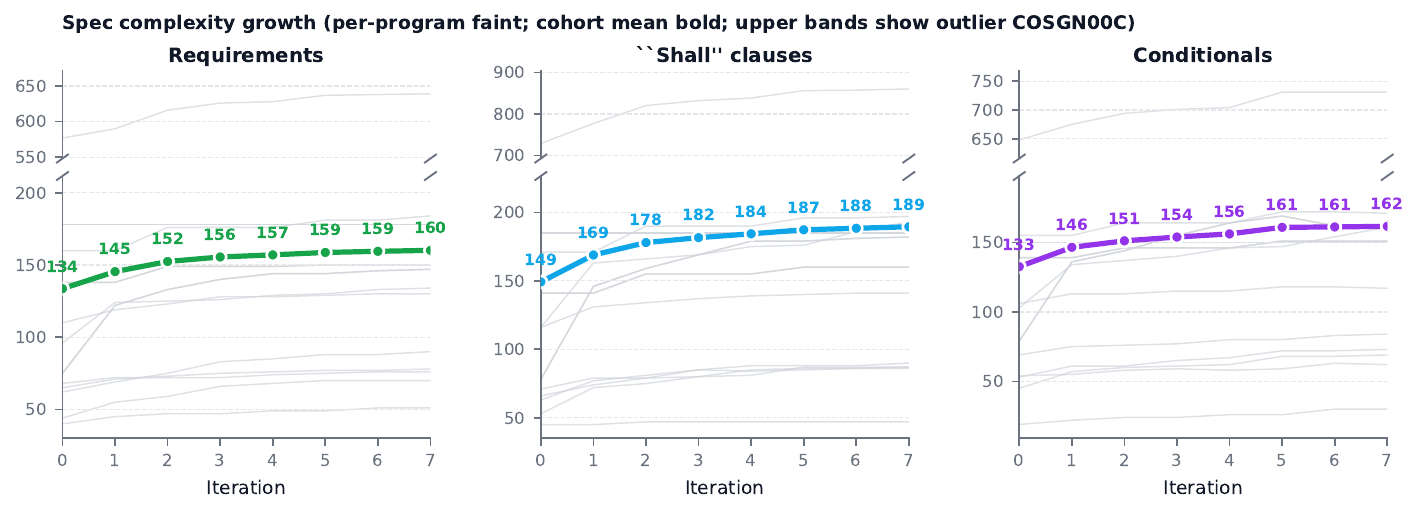}
  \caption{Specification size over revisions. Faint traces show individual programs and bold traces show program means. Requirements, ``shall'' clauses, and conditionals grow most rapidly during the early revisions; later revisions continue to add content.}
  \label{fig:brief-spec-complexity}
\end{figure}

\begin{figure}[h]
  \centering
  \includegraphics[width=0.95\linewidth]{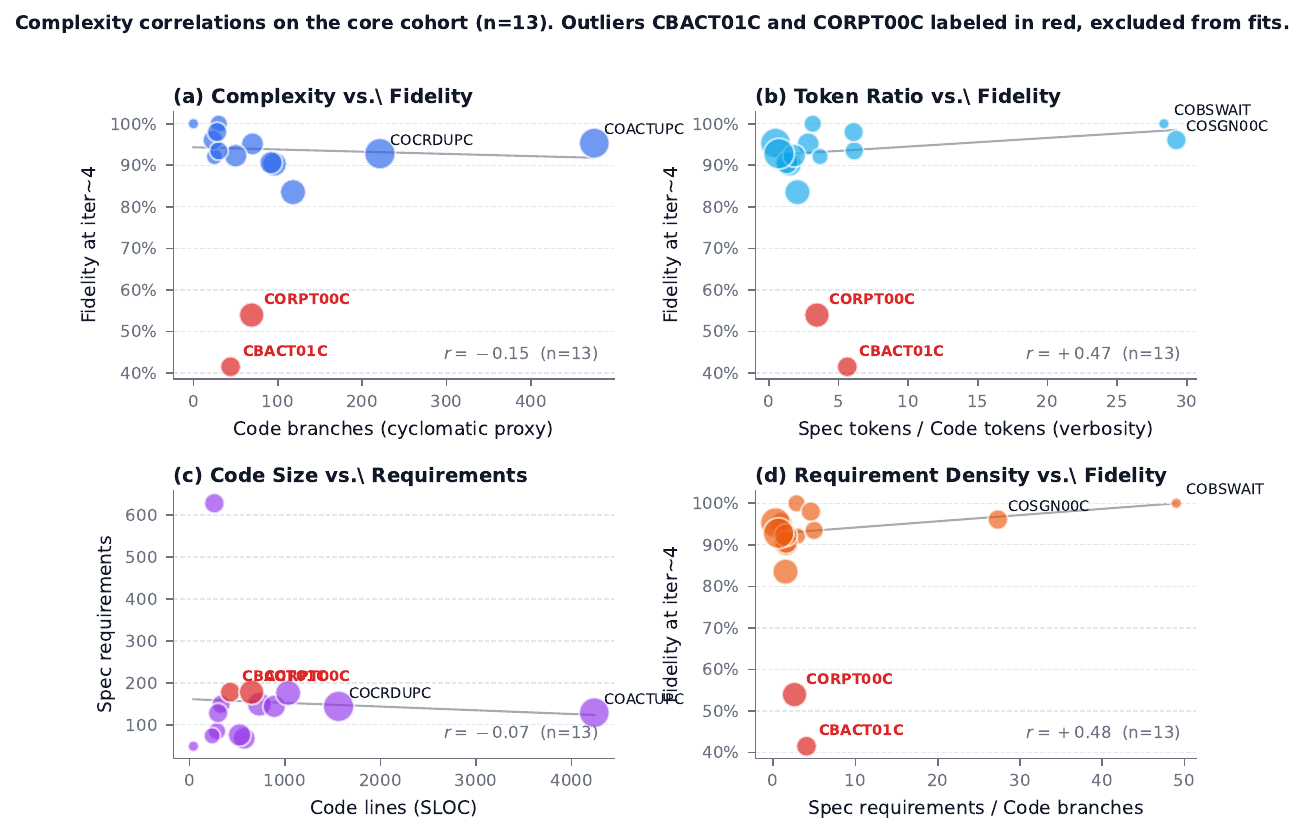}
  \caption{Structural measures and fidelity. The red-labelled programs \textsc{CBACT01C} and \textsc{CORPT00C} are excluded from the 13-program regression fits; bubble area represents probe count. Branch count has little linear association with fidelity in this sample, while specification/code ratios show positive associations. These comparisons do not establish what causes specification growth.}
  \label{fig:brief-complexity-correlations}
\end{figure}

\FloatBarrier
\section{Neurosymbolic-evaluation figures}
\label{app:brief-neurosymbolic}

This appendix carries the per-mixture dashboard and heatmap views for the three symbolic mixtures (Figures~\ref{fig:brief-dashboard-balanced-sym}--\ref{fig:brief-heatmap-pure-cfg}). All panels aggregate on the fixed 11-program intersection cohort; the pure-LLM reference is on the 13-program core cohort of Table~\ref{tab:brief-headline}. The all-mixture fidelity trajectory (Figure~\ref{fig:brief-neurosymbolic-trajectory}) is shown side-by-side with the generator-sweep trajectory in the body as Figure~\ref{fig:trajectories-side-by-side}; the cross-distribution transfer figure follows below.

\begin{table}[h]
  \centering
  \footnotesize
  \setlength{\tabcolsep}{4pt}
  \caption{Estimated balance levels $\hat F^\dagger_k=\hat\pi_k/(\hat\pi_k+\hat r_k)$ for the symbolic mixtures on the 11-program cohort. A balance level describes the measured repair and regression rates of one transition; changing rates do not establish a stable plateau. The direct reference uses the 13-program core.}
  \label{tab:brief-neurosymbolic-pi-r}
  \begin{tabular}{lccccccc}
    \toprule
    Mixture                                                   & $\hat F^\dagger_{0\to 1}$ & $\hat F^\dagger_{1\to 2}$ & $\hat F^\dagger_{2\to 3}$ & $\hat F^\dagger_{3\to 4}$ & $\hat F^\dagger_{4\to 5}$ & $\hat F^\dagger_{5\to 6}$ & $\hat F^\dagger_{6\to 7}$ \\
    \midrule
    pure CFG ($\alpha{=}0$, $\beta_{\mathrm{cfg}}{=}1$)       & 0.774                     & 0.828                     & 0.643                     & 0.887                     & 0.969                     & 0.676                     & 0.948                     \\
    half LLM     ($\alpha{=}0.5$, $\beta_{\mathrm{cfg}}{=}1$) & 0.680                     & 0.614                     & 0.827                     & 0.824                     & 0.852                     & 0.786                     & 0.733                     \\
    balanced sym ($\alpha{=}0$, $\beta{=}(.34,.33,.33)$)      & \textbf{0.886}            & \textbf{0.852}            & \textbf{0.847}            & \textbf{0.854}            & \textbf{0.837}            & \textbf{0.829}            & \textbf{0.956}            \\
    \midrule
    pure LLM ($\alpha{=}1$, 13-prog core)                     & 0.924                     & 0.945                     & 0.931                     & 0.930                     & 0.948                     & 0.947                     & 0.929                     \\
    \bottomrule
  \end{tabular}
\end{table}

\begin{figure}[h]
  \centering
  \includegraphics[width=0.97\linewidth]{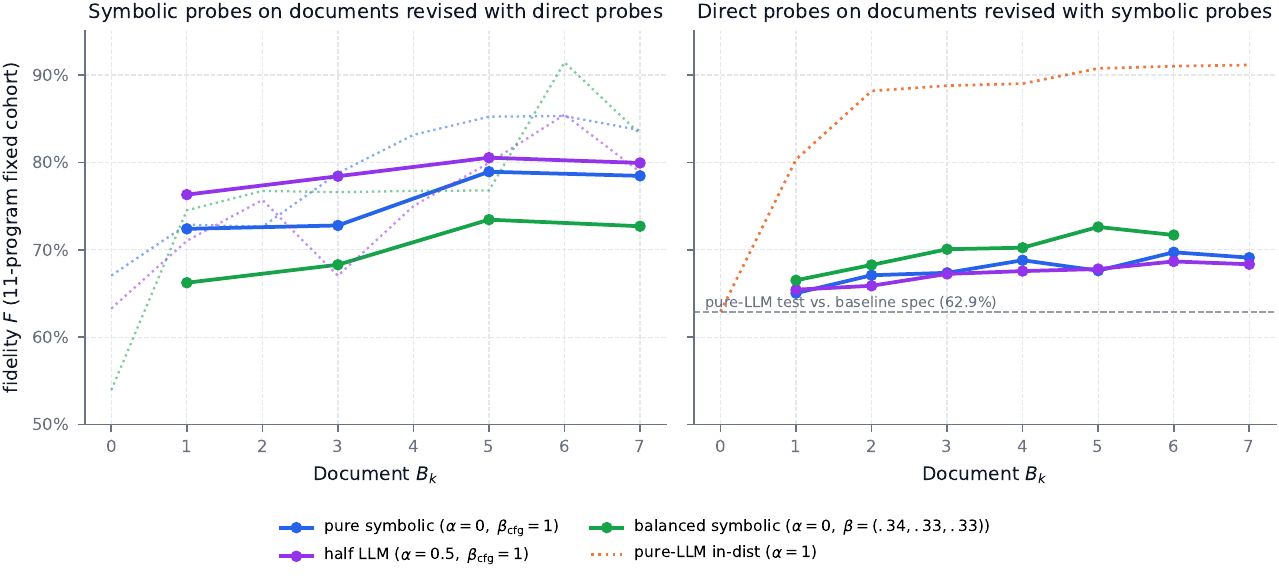}
  \caption{Transfer between probe distributions. Left: symbolic test probes evaluated on documents revised with direct LLM probes (solid), with each symbolic mixture's curriculum trajectory dotted. Right: direct LLM test probes evaluated on documents revised with symbolic probes (solid), with the direct held-out trajectory dotted. Both panels use the same eleven programs and document indices. Gains transfer in both directions, with smaller gains on direct probes after symbolic revision.}
  \label{fig:brief-cross-distribution}
\end{figure}

\begin{figure}[h]
  \centering
  \includegraphics[width=0.97\linewidth]{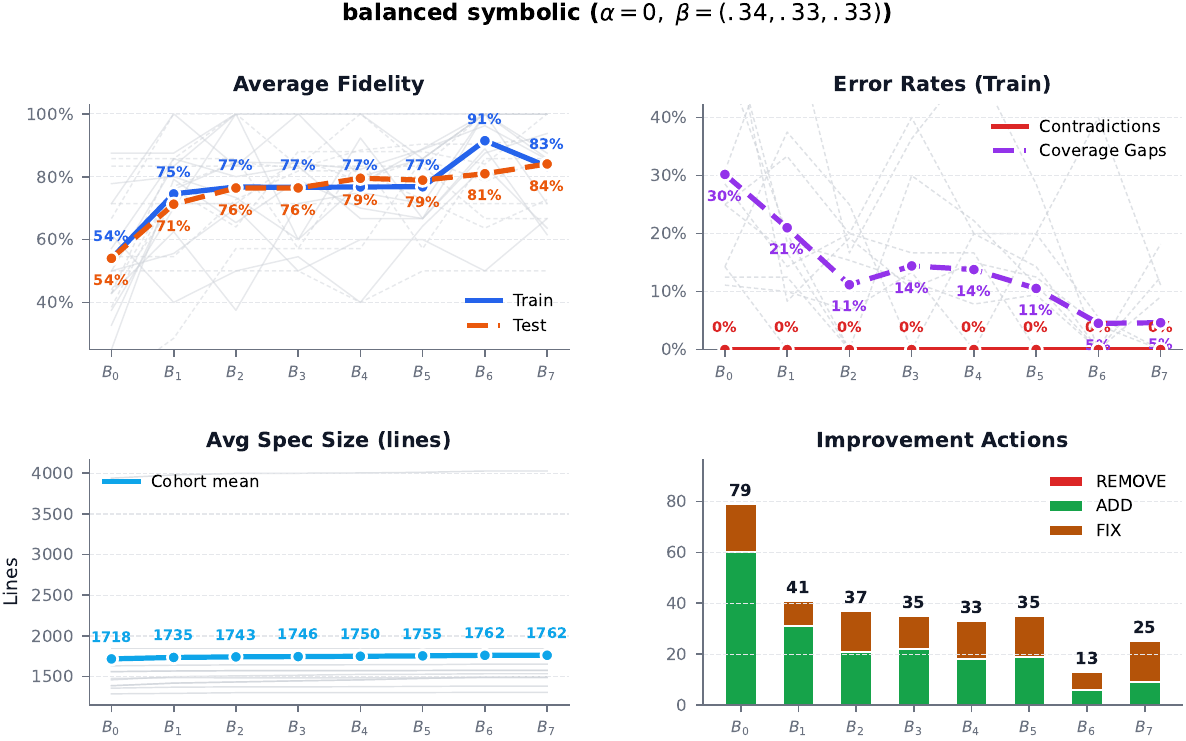}
  \caption{Balanced graph mixture, with fidelity, stored error diagnostics, document size, and proposed actions. The mixture asks about control, data, and call-graph facts. Its probe bank differs from the other mixtures, so gains are interpreted within the run. Stored error diagnostics may overlap.}
  \label{fig:brief-dashboard-balanced-sym}
\end{figure}

\begin{figure}[h]
  \centering
  \includegraphics[width=0.97\linewidth]{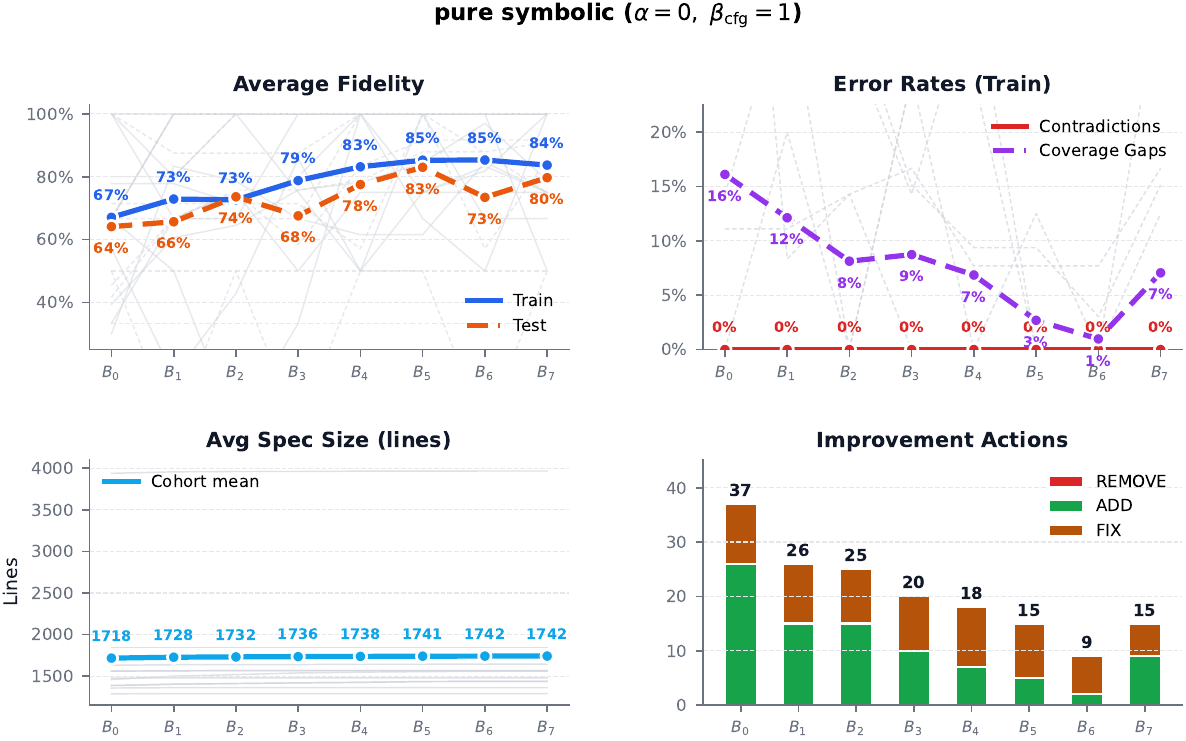}
  \caption{Pure-CFG mixture ($\alpha{=}0$, $\beta_{\mathrm{cfg}}{=}1$). Held-out fidelity falls from 0.83 at $B_5$ to 0.73 at $B_6$ and recovers to 0.80 at $B_7$. The decline is concentrated in three credit-card and bill-payment programs (Figure~\ref{fig:brief-heatmap-pure-cfg}).}
  \label{fig:brief-dashboard-pure-cfg}
\end{figure}

\begin{figure}[h]
  \centering
  \includegraphics[width=0.97\linewidth]{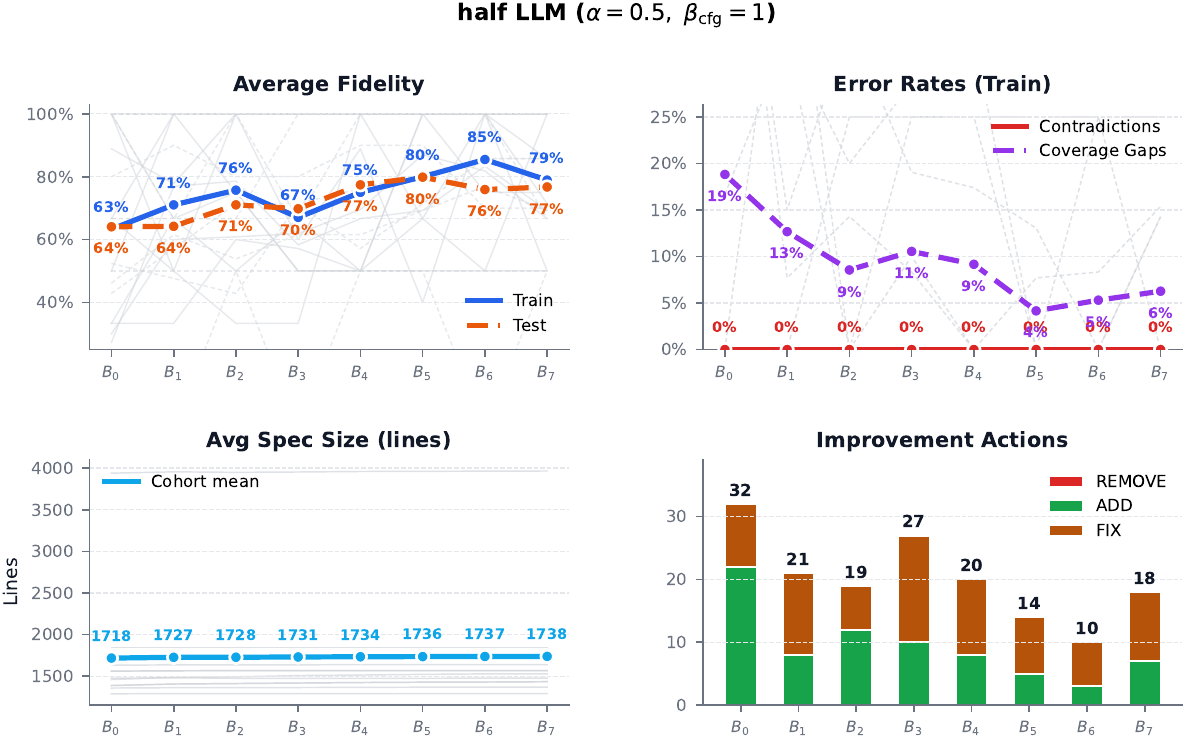}
  \caption{Half-LLM mixture, with the same four-panel layout as the other dashboards. The panels show score changes, stored error diagnostics, document size, and proposed actions over eight evaluations. Error diagnostics may overlap; score changes are specific to this mixture's probe bank.}
  \label{fig:brief-dashboard-half-llm}
\end{figure}

\begin{figure}[h]
  \centering
  \includegraphics[width=0.85\linewidth]{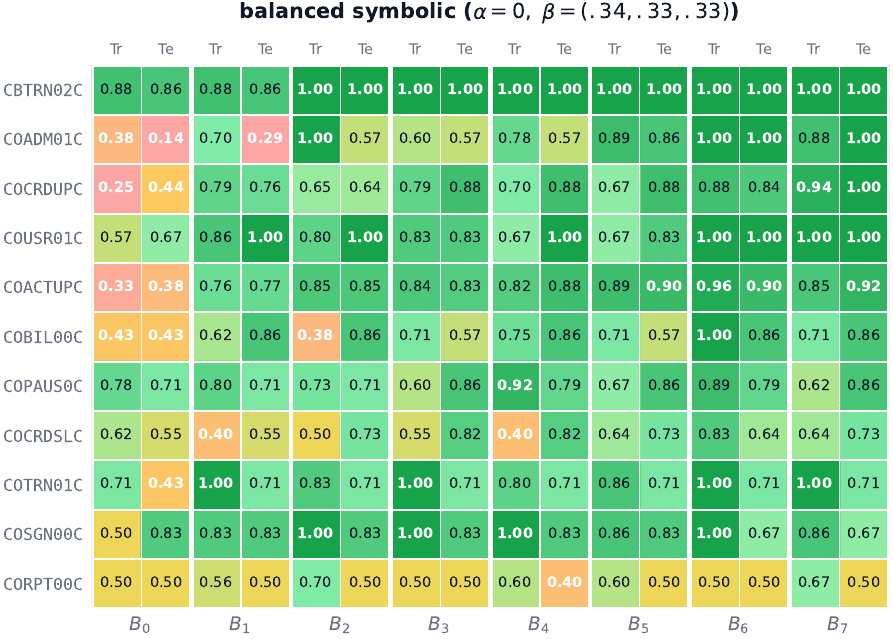}
  \caption{Program scores for the balanced graph mixture on the eleven-program intersection, with curriculum and held-out results side by side. Each column pair evaluates the same document $B_k$.}
  \label{fig:brief-heatmap-balanced-sym}
\end{figure}

\begin{figure}[h]
  \centering
  \includegraphics[width=0.85\linewidth]{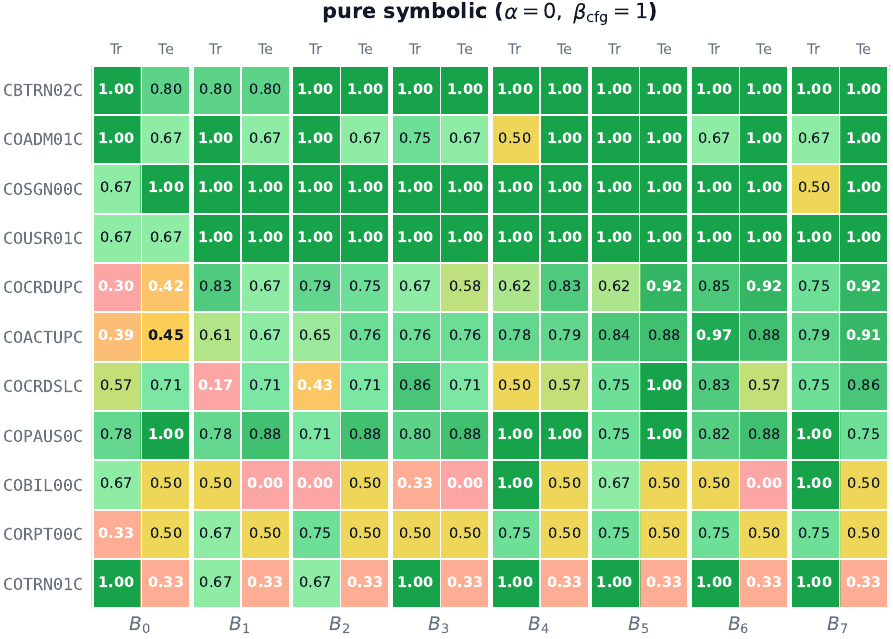}
  \caption{Pure-CFG program scores. The decline at $B_6$ affects \textsc{COBIL00C}, \textsc{COCRDSLC}, and \textsc{COPAUS0C}; the first two recover at $B_7$. Column pairs show curriculum and held-out scores for the same document.}
  \label{fig:brief-heatmap-pure-cfg}
\end{figure}

\begin{figure}[h]
  \centering
  \includegraphics[width=0.85\linewidth]{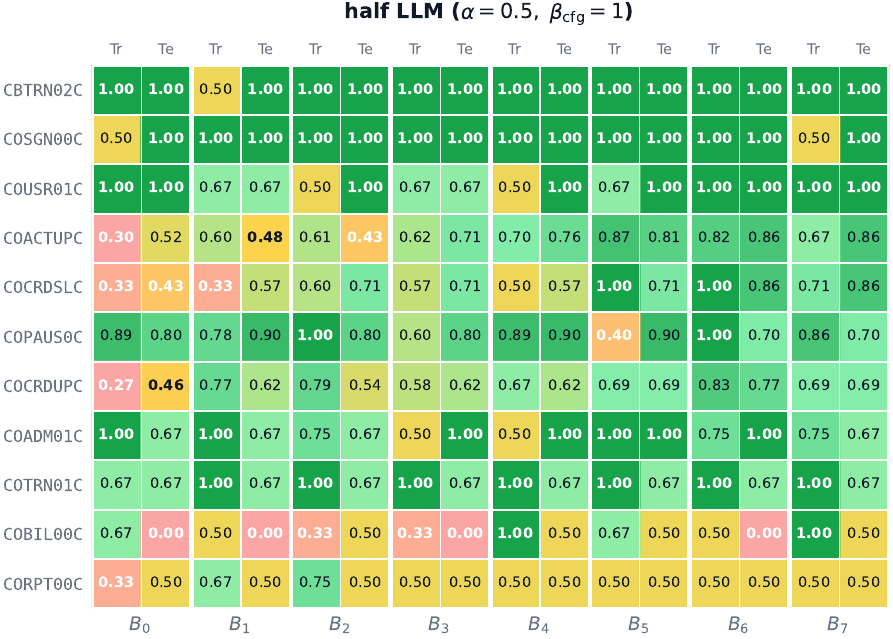}
  \caption{Program scores for the half-LLM mixture, with curriculum and held-out results side by side. Each cell reports agreement on that program's questions; differences from other mixtures reflect different probe banks as well as revised documents.}
  \label{fig:brief-heatmap-half-llm}
\end{figure}

\FloatBarrier
\section{Channel construction and emitted fact types}
\label{app:channel-construction}

\textbf{Construction.} The static-analysis pipeline expands copybooks, parses
each program, and builds the graphs used for extraction.
The annotated control-flow graph (ACFG) has decision and action nodes.
Actions carry read and write sets, and each decision has a resolved
true-branch successor. The data-flow graph (DFG) records def-use chains.
The extractor uses these chains to compute backward slices from observable
sinks: screen output, persisted business records, and cross-system messages.
The system-dependence graph (SDG) combines control dependence, data dependence,
and interprocedural analysis.

The CallGraph channel selects only the interprocedural \texttt{CALL},
\texttt{XCTL}, and \texttt{LINK} edges from the SDG.
For dynamic calls, the extractor scans the call site for a table-indexed or
context-variable program name and tags the target as dynamic. This resolution
is lexical; the extractor does not use data-flow analysis to resolve the target.

\textbf{Emitted fact types.} The extractor tags each fact with one type,
and a single function maps that type to one channel.
Table~\ref{tab:emitted-types} lists the types that the pipeline emits.
The emitted record types form a disjoint partition. Facts assigned to
different channels can still describe overlapping program behaviour.

\begin{table}[h]
  \centering
  \footnotesize
  \setlength{\tabcolsep}{4pt}
  \renewcommand{\arraystretch}{1.1}
  
  \caption{Every fact type the extractor emits, by channel, and the structure it is read from.}
  \label{tab:emitted-types}
  \begin{tabularx}{\linewidth}{@{}llX@{}}
    \toprule
    \textbf{Channel} & \textbf{Type} & \textbf{What the fact records} \\
    \midrule
    CFG & \texttt{guard} & a decision node together with the action on its resolved true branch: what observable effect follows when the condition holds \\
    \midrule
    DFG & \texttt{composition} & an observable output assembled from several source fields along a slice, with the transformation (for example a delimited \texttt{STRING}) \\
    DFG & \texttt{truncation} & a value shortened by reference modification on its way to an observable sink \\
    DFG & \texttt{projection} & a plain \texttt{MOVE} of an input field onto an observable output field with no transformation \\
    DFG & \texttt{sql\_operation} & an \texttt{INSERT}, \texttt{UPDATE}, or \texttt{DELETE} on a persisted business table, phrased as a business event \\
    \midrule
    CallGraph & \texttt{inter\_program\_call} & a \texttt{CALL}, \texttt{XCTL}, or \texttt{LINK} to a literal target program \\
    CallGraph & \texttt{dynamic\_call} & a \texttt{CALL}, \texttt{XCTL}, or \texttt{LINK} whose target is a table-indexed or context-variable program name, resolved lexically \\
    \bottomrule
  \end{tabularx}
\end{table}

\textbf{Sampling.} Within a channel, the sampler shuffles candidate facts and
selects without replacement, drawing at most one probe per SDG region
(a paragraph for intraprogram facts, a call site for interprogram facts)
per iteration. This limits clustering within a paragraph.

Before sampling, the observability filter excludes facts whose answers
depend on platform details such as file-access ordering, queue operations,
buffers, paragraph identifiers, or record layouts. Persisted business data
remains in scope: updating an account balance is observable business
behaviour, while the storage mechanism used for that update is not.

\FloatBarrier
\section{Manual evaluation: protocol and results}
\label{app:manual-eval}

Five human checks address different questions: a source audit examines
reference answers, verdicts, and edits; a blind review checks paired-probe
equivalence; an external audit supplies requirement-level defect labels;
and two later expert panels assess probe relevance and reported findings.
The later panels use additional CardDemo runs and do not re-evaluate the
main revision trajectory.

\textbf{Check 1: expert audit of probes, verdicts, and edits.}
We fixed the sample before annotation, drawing 120 held-out CardDemo probes
with seed 29844. Stratification used the final-document verdicts: 100 \emph{agree},
10 \emph{contradict}, and 10 \emph{gap}, with every nonempty
program-by-verdict cell represented.
We also sampled 30 applied edits: 20 additions, 7 fixes, and 3 removals.
For each probe, the expert assessed the reference answer against the COBOL
source and judged whether the model's verdicts at iteration 0 and the final
iteration were correct. For each edit, the expert assessed faithfulness to
the code and whether an incorrect claim had been introduced.

The expert worked from line-numbered source and specification excerpts,
without access to the pipeline. The pack showed the recorded model verdicts
and contained AI-drafted judgments that the expert could confirm or override.
This was AI-assisted human adjudication. The expert changed the draft
judgment on 25 of the 120 iteration-0 verdict cells, 7 of the 120
final-iteration cells, 0 of the 120 reference-answer cells, and 3 of the
60 edit cells (two judgments per edit). Every change included a rationale
identifying the label the
expert considered correct.

The expert confirmed 119 of 120 reference answers; the remaining answer named
a prompt that the calling path cannot reach. Verdicts were correct on 93 of
120 initial-document evaluations and 109 of 120 final-document evaluations.
Across revisions, 23 verdicts moved from wrong to right and seven the other
way. Of the 30 edits, 25 were faithful and introduced no incorrect claim.
The remaining five introduced an incorrect claim, including four that mixed
correct and incorrect additions.

\textbf{Check 2: equivalence re-check.}
A second reviewer assessed 53 pairs of informalized probes with the
comparator's verdicts hidden: all 3 pairs it called different and 50 it
judged equivalent. The reviewer assessed question equivalence and answer
equivalence separately. Agreement with the comparator was 51 of 53
(Appendix~\ref{app:brief-stability}). This reviewer did not participate
in the source audit.

\textbf{Check 3: independent requirements audit.}
Professional analysts engaged by the owner of \emph{System B} reviewed the
derived requirements using their own protocol, with no contact with our work.
They flagged 72 of 396 requirements as defective: 20 as wrong and 52 as partial.
We compared their findings with the requirements cited by failing probes.
At the level of program-scoped requirement instances, 43 of 872 instances
were in the flagged set. Failing probes cited 33 of these 43 instances
(77 percent), compared with 374 of the 829 unflagged instances (45 percent).
The odds ratio was 4.0, with Fisher exact $p = 0.0001$.

Within the matched audit subset, 23 of 28 incomplete requirements were
cited by a probe with a gap verdict. The association between the auditors' ``wrong'' label and
our contradiction verdicts was positive (4 of 11 matched wrong requirements,
36 percent, against a
20 percent base rate), but not significant at this sample size ($p = 0.25$).
The audit was not exhaustive, so an unflagged requirement is not necessarily
correct and the unflagged rate cannot be interpreted as a false-positive rate.
The comparison uses different units: auditors judged individual requirement
statements, whereas probes were answered from program-specific documents.
A later inspection found that document assembly omitted some requirement
blocks in 18 of the 21 judged programs, retaining 64.9 percent of the available
text by character count. The reported association therefore concerns these
incomplete excerpts; its strength on complete documents has not been measured.

\textbf{Check 4: probe-quality panel.}
Three mainframe modernization specialists reviewed 24 generated probes from
a later CardDemo run using Claude Opus 5 as generator and judge.
Eight probes were drawn randomly from each of three business functions,
with no filtering by relevance or subsequent verdict. The sample covered
ten programs and contained 19 code-derived and five requirement-derived
questions. JCL questions were excluded; the eligible pool contained 930
code-derived and 206 requirement-derived questions.

Each reviewer saw the question, its code-derived answer, and cited source
lines. They classified the question as business behaviour, technical noise,
or unclear, and could flag an incorrect answer. One demonstration question
was discussed jointly; subsequent ratings were individual. We exclude that
question from all independent-agreement statistics.
On the 19 remaining questions where all three gave a binary decision,
16 were unanimous, raw pairwise agreement was 0.895, and Fleiss
$\kappa=0.77$. Treating unclear as a third category gives $\kappa=0.64$
on 22 questions. Across all 24 questions, the majority classified 14 as
business behaviour and eight as noise, with one unclear judgment and one tie.

A rubric rater (Opus 5) separately classified each question and answer
without reading source code. It agreed with 20 of 21 independent, decided
majority judgments (Cohen's $\kappa=0.89$). Including the shared demonstration
gives 21 of 22 agreements; that extra agreement is not independent evidence.
The rater repeated its label on 23 of 24 questions across five reruns.
The one decided disagreement concerned whether a menu's accepted option
range should survive a change of user interface.

Reviewers flagged six answers as wrong. A separate model-based code review
rejected those six flags and identified an unflagged boundary error
(500 versus 501). These follow-up decisions are model assessments, not
additional human correctness labels. Missing source context and imprecise
anchors contributed to several flags. The study supports the relevance
rubric on this small sample while exposing limits in the review evidence.

\textbf{Check 5: expert review of findings.}
Five specialists, including the three from check 4, individually reviewed
24 findings from another CardDemo run over 18 COBOL programs and three
business functions. This run generated 354 questions from control-flow
decisions and observable sinks, using Opus 5 as generator, validator, and
judge. Each question was checked against an extracted business-rules
document (stage 1) and mapped requirements (stage 2).
The run flagged 92 questions, yielding 96 stage-specific records.

Before sampling, a separate model-based code-reading verifier assessed the answers and
findings against source code and the exact document excerpts the judge saw.
Seven records with answers it considered incorrect were excluded, and eight
duplicate records were collapsed, leaving 81. A fixed-seed stratified draw
(seed 20260921) selected 24: nine, six, and nine per business function,
with four stage-1 and 20 stage-2 records. Selection was proportional across
stage and run verdict within each function, with at least one stage-1 record
per function. The sample contained 16 missing, four contradicted, two review,
and two dropped verdicts. Review means the comparison was undecided;
dropped means no requirement was mapped.
The verifier's genuine/false-positive labels were not used in the draw.

Review cards showed the question, code-derived answer, document answer,
quoted passage, and links to the source and document excerpts.
The run's comparison was available in a collapsed panel.
Reviewers chose genuine, false positive, or unable to determine, with optional
reasons and comments. The separate warm-up card was excluded throughout.
The verifier, run through OpenAI Codex before the human review, supplies a
model comparison rather than a human reference label; its underlying model
snapshot is not recorded in the report.

The majority called 15 findings genuine and seven false positives; two tied.
On the 18 findings all five reviewers decided, six were unanimous, raw
pairwise agreement was 0.611, and Fleiss $\kappa=0.13$.
Counting unable to determine as a third category gives $\kappa=0.04$ on
all 24. The verifier agreed with the majority on 15 of 22 decided findings
(Cohen's $\kappa=0.29$). Of its 15 genuine labels, the majority confirmed
11, rejected three, and tied on one. These are results on a screened sample
of flagged records, not estimates of end-to-end precision or recall.

A descriptive grouping of the reviewed findings helps explain the variation.
Seven findings about record content, editable fields, selection precedence,
and lookup conditions received 34 genuine votes out of 35, with six unanimous
findings. Ten navigation and screen findings split evenly: 24 genuine votes,
24 false-positive votes, and two abstentions. Reviewers differed on which
interface behaviour a modernization specification should retain.
The four misstatement findings received ten genuine votes, seven false
positives, and three abstentions. On two of these records, a failed-lock
path's final state was traced in the source, but the claimed display of a
success message remained disputed. On other cases, reviewers assessed
document completeness while the verifier assessed whether the reported
defect had been correctly derived. Both questions need to be asked explicitly.
These groups describe this sample; multiple votes on a finding are not
independent defects, and records can share programs or questions.

\textbf{Who rated what.}
Checks 1 and 2 used different reviewers on different material; check 3 used
an external audit. No human inter-rater statistic is available for those
checks. Checks 4 and 5 used three and five human raters, respectively, with
overlapping membership. Table~\ref{tab:kappa} distinguishes human agreement
from model--human comparisons. Anonymous item-level ratings and a script
for recomputing the panel statistics are included in the source package.

\begin{table}[h]
  \centering
  \footnotesize
  \setlength{\tabcolsep}{4pt}
  \renewcommand{\arraystretch}{1.1}
  
  \caption{Agreement in the human checks. Human panels use Fleiss' $\kappa$; pairs use Cohen's $\kappa$. The source-audit rows use AI-assisted human adjudication. The shared probe-quality demonstration is excluded from both panel comparisons.}
  \label{tab:kappa}
  \begin{tabularx}{\linewidth}{@{}Xrrr@{}}
    \toprule
    \textbf{Review task} & $n$ & \textbf{Agreement} & $\boldsymbol{\kappa}$ \\
    \midrule
    Three reviewers, probe relevance & 19 & raw $0.895$ & $0.77$ \\
    Probe majority against rubric rater & 21 & $20/21 = 0.952$ & $0.89$ \\
    Five reviewers, finding validity & 18 & raw $0.611$ & $0.13$ \\
    Finding majority against verifier & 22 & $15/22 = 0.682$ & $0.29$ \\
    Equivalence comparator against human (independent) & 53 & $51/53 = 0.962$ & $0.73$ \\
    Judge against expert, three-way verdict, iteration 0 & 120 & $93/120 = 0.775$ & $0.63$ \\
    Judge against expert, three-way verdict, final iteration & 120 & $109/120 = 0.908$ & $0.65$ \\
    Judge against expert, defect / no defect, iteration 0 & 120 & $105/120 = 0.875$ & $0.75$ \\
    Judge against expert, defect / no defect, final iteration & 120 & $111/120 = 0.925$ & $0.70$ \\
    \bottomrule
  \end{tabularx}
\end{table}

At the final iteration, 100 of the 120 judge verdicts and 105 of the
120 adjudicated labels are \emph{agree}. This imbalance raises chance
agreement to 0.74, compared with raw agreement of 0.91, and limits kappa.
Of the 27 three-way disagreements at iteration 0, 12 are between
\emph{contradict} and \emph{gap}. Both raters identify a failure in those
cases but assign it to different classes. Combining these classes into
defect versus no defect raises the iteration-0 kappa to 0.75.
One rationale in the pack is ambiguous, recorded as ``contradict (or gap)'';
the iteration-0 kappa is 0.628 under either interpretation.

\FloatBarrier
\section{Template against LLM informalization}
\label{app:template-study}

\textbf{Design.} We generated two versions of each probe from the same
90 CFG guard facts. Each fact contains a guard condition $G$ and its
true-branch action $T$, extracted deterministically from the annotated
control-flow graph. Claude Opus 4.6 judged both versions against the
corresponding program's initial specification $B_0$.
The template version inserts the fact fields verbatim into the question
``Under what condition does the following occur: $T$?'' and the reference
answer ``This occurs when $G$.'' This step uses no LLM.
The LLM version uses the frozen informalizations from the stability check
in Appendix~\ref{app:brief-stability}.
The 90 facts span 19 programs. They come from the 100-fact stability sample,
excluding ten facts whose identifiers collided across programs or whose
program had no $B_0$ mapping.

For example, a template asks when \texttt{MOVE 'Y' TO WS-ERR-FLG} occurs,
whereas the corresponding informalized question asks what happens when a
processing date is blank. Both refer to the same extracted guard, but their
wording asks the specification to recover different information.

\textbf{Results.} Fidelity is 22.2 percent (20 of 90) for the template arm
and 44.4 percent (40 of 90) for the informalizer. Coverage gaps account for
61 versus 37 probes, and contradictions for nine versus thirteen.

Thirteen facts pass in both arms and 43 fail in both. The template fails
where the informalizer passes on 27 facts, and the reverse occurs on seven
(exact two-sided McNemar $p=0.00082$). A scanner for COBOL identifiers,
CICS attention keys, analysis-helper names, and the extractor's banned
platform terms flags all 90 template probes and no informalized probes.

\textbf{Interpretation.} Templates preserve source identifiers that are absent
from business requirements, which can make an otherwise relevant rule harder
to locate. The question pair may also change direction: asking which condition
causes an action is different from asking what follows from a condition.
The informalizer can add or omit context. The paired agreement difference
therefore measures sensitivity to rendering; it does not show that every
higher-scoring informalization preserves the original fact.

\FloatBarrier
\section{Further systems: details}
\label{app:further-systems}

\textbf{System A} is a proprietary production COBOL system with 352 programs,
418,992 lines of program source, and 61,431 lines of copybooks.
The pipeline generated 23,526 probes and scored the derived behavioural
specification in a single pass. Pooled fidelity was $F = 0.857$, with
contradiction 0.074 and coverage gap 0.069; the per-program mean was 0.866
and the median 0.887.
A further abstraction step produced a second document for the same programs,
scoring $F = 0.426$ (contradiction 0.148, gap 0.426).
Most of the fidelity loss comes from the increase in coverage gaps, indicating
that the later document omits behaviour covered by the first.
Only the measurement pipeline is installed in the owner's environment,
so no revision iteration was run.
The released per-program table uses the names P001 through P352 and reports
source size, probe count, and fidelity.

\textbf{System B} contains 32 PL/1 programs and 57,187 lines. A separate
1,756-probe measurement gives fidelity 0.584, contradiction 0.056, and gap 0.360.
For revision, the same 1,668 held-out questions score both documents:
agreements rise from 960 to 1,427, contradictions fall from 98 to 67, and gaps
fall from 610 to 174. Thus fidelity rises from 0.576 to 0.856 in one revision.
The independent audit concerns this system's requirements. Fourteen delegation
wrappers have no captured requirements and are excluded from the requirement-level
concordance analysis.

\textbf{PetClinic} supplies a further public revision case in Java: 25 mapped
source files, about 1.8\,k lines, grouped into four functional areas.
The same 203 questions score four document versions at 0.744, 0.956, 0.970,
and 0.990, using Opus~4.6. This is a separate run with its own probe bank;
it is not evidence that scores are comparable across languages.

\textbf{Redis} is evaluated against the project's official per-command
documentation. The source is pinned at release 7.2.5 and the documentation
at a separately pinned commit, covering six modules and about 6,724 lines.
The source commit is \texttt{f60370ce28b946c1146dcea77c9c399d39601aaa};
the documentation commit is \texttt{3541d0e20cc4bb7873bdbf51a7717757b806577f}.
The input documents concatenate command prose and return-value documentation
by command family. The documentation repository has no matching 7.2 tag,
so this pin approximates the release era and may contain version differences.
Of 600 probes, 242 are coverage gaps and 13 are contradictions, the lowest
contradiction count in these evaluations. We inspected all 13; most concern
outdated or simplified descriptions. This case evaluates the method against
human-authored documentation. Its coverage gaps should be interpreted in
light of the documentation's intended scope.

\textbf{Language-specific prompts.}
For the direct LLM channel, we adapt probe generation through a dialect prompt
of about one page. It combines the shared observability rule with instructions
about platform details to exclude. Examples include attention keys and BMS
maps for CICS, and function-selector codes and record framing for PL/1.
The relevance panel shows that these instructions do not remove all
technical noise. CardDemo and \emph{System A} were judged with Claude
Opus 4.6; \emph{System B} and Redis with Opus 4.8.

\FloatBarrier
\section{Prompt Templates}
\label{app:prompts}

The six listings below document the direct-probe generation, sizing,
validation, specification answering, answer comparison, and revision
interfaces. Each is a Markdown template with runtime placeholders such as
\texttt{\{source\_code\}} and \texttt{\{num\_questions\}}.
They do not cover every experimental variant: graph informalization,
paired-probe equivalence checks, free-form revision, later review pipelines,
and language-specific adaptations use additional prompts or instructions.
The printed templates are retained as experimental artifacts.

\lstset{
  basicstyle=\ttfamily\footnotesize,
  breaklines=true,
  breakatwhitespace=false,
  columns=fullflexible,
  frame=single,
  framesep=3pt,
  xleftmargin=4pt,
  xrightmargin=4pt,
  keepspaces=true,
  showstringspaces=false,
  extendedchars=true,
  inputencoding=utf8,
  literate={—}{{---}}1 {→}{{$\to$}}1 {’}{{'}}1 {“}{{``}}1 {”}{{''}}1
           {é}{{\'{e}}}1 {á}{{\'{a}}}1 {í}{{\'{i}}}1 {ó}{{\'{o}}}1 {ú}{{\'{u}}}1
           {…}{{\ldots}}1 {–}{{--}}1 {•}{{\textbullet}}1 {≥}{{$\ge$}}1 {≤}{{$\le$}}1,
}

\subsection{Q\&A generation: \texttt{qa\_generation.md}}
\label{app:prompt:qagen}

Instantiates the direct probe generator $Q$. The prompt requests questions
in business vocabulary, excludes specified platform terminology, and asks
for eight behaviour categories with at least 15\% negative probes.
It requests structured JSON. These are generation instructions, not
guarantees that every returned probe meets the requested vocabulary or mix.
\begin{lstlisting}[language={}]
You are a software behavior analyst. The source code below is LEGACY MAINFRAME CODE (COBOL/CICS). These questions will be used to verify whether a MODERNIZATION SPECIFICATION correctly captures the business behavior of this code. The spec is written for building a modern replacement application — it describes business rules and user workflows, not platform mechanics.

RULES:
- Questions must be about WHAT the program does from a business perspective, not HOW the legacy code is structured.
- A business analyst writing a modernization spec should be able to understand and answer the question.
- Focus on behaviors that MUST be preserved in a modernized system: validation rules, business logic, error handling visible to users, data transformations, and workflow navigation.
- Do NOT ask about line numbers, variable names, function call order, types, or syntax.
- Each answer must be derivable solely from the code's semantics.
- NEVER use implementation or platform terms in questions OR answers. This includes but is not limited to:
  Language constructs: PERFORM, MOVE, EVALUATE, INSPECT, STRING, UNSTRING, PIC, COMP-3, 88-level, WORKING-STORAGE, LINKAGE SECTION, paragraph names
  Middleware/platform: CICS, COMMAREA, EIBCALEN, DFHRESP, DFHCOMMAREA, EIBTRNID, BMS, SEND MAP, RECEIVE MAP, XCTL, CICS RETURN, SYNCPOINT, EXEC SQL, EXEC CICS
  File system: VSAM, KSDS, ESDS, QSAM, file status codes (e.g., "status 00"), READ INTO, WRITE FROM, DD names
  Terminal/UI: F1, F3, F4, F5, F7, F8, F12, PF keys, AID keys, ENTER key, 3270, BMS maps — these are legacy terminal keybindings, not business actions
  Internal names: any WS- prefix, any variable/copybook/paragraph name from the source
- Instead, use business-domain language:
  "XCTL to COMEN01C" → "navigates to the main menu"
  "SEND MAP" → "displays the screen"
  "SYNCPOINT" → "commits the changes"
  "low-values" → "cleared/empty"
  "COMMAREA" → "session data"
  "file status 23" → "record not found"
  "F3 pressed" → "the user requests to exit/go back"
  "F5 pressed" → "the user confirms the action"
  "F12 pressed" → "the user cancels the operation"
  "ENTER key" → "the user submits the form"
  "F7/F8" → "the user pages up/down"
- Focus on USER-VISIBLE behavior: what the user sees, what messages are displayed, what happens when buttons/keys are pressed, what fields are validated, what navigation occurs.
- Focus on BUSINESS RULES: validation logic, required vs optional fields, allowed values, error messages shown to users, confirmation flows.
- AVOID asking about:
  - Array/buffer sizes, record layouts, or data structure internals
  - Processing counters or internal accumulators not visible in output
  - The order in which internal data structures are populated
  - Internal variable initialization sequences
  - Program startup banners, execution-complete messages, or diagnostic log output (e.g., "END OF EXECUTION OF PROGRAM X") — these are infrastructure artifacts, not business behavior
  - Internal file open/close sequences or the order queues are initialized
  - Return codes, condition codes, or abend codes
- DO ask about:
  - For ONLINE programs: what happens when the user exits, submits, confirms, cancels, or pages through data; what validation is performed on each input field; what messages are shown for invalid input; what data is shown on screen; what navigation occurs after each action; required vs optional fields
  - For BATCH programs: what input is consumed, what output is produced, what records are processed, what happens when input is invalid or missing, what summary or report is generated
  - For ALL programs: what happens when a record is not found or already exists, what error handling is visible to the user or in output, what business rules govern processing

CATEGORIES (generate at least one question per applicable category):
- precondition: Under what conditions does behavior X occur?
- computation: How are values derived?
- branching: What varies by condition?
- guard: What prevents or forces behavior?
- output: What data is produced?
- dependency: Do values influence each other?
- negative: Things the code does NOT do (at least 15% of questions must be negative probes)
- boundary: Behavior at domain edges

DIFFICULTY LEVELS:
- easy: Single condition, direct read
- medium: Multiple conditions, requires combining rules
- hard: Cross-function data flow, nested conditions, reasoning about interactions

TARGET: Generate approximately {num_questions} Q&A pairs distributed across categories.

SOURCE CODE:
{source_code}

Respond with a JSON array. Each element:
{
  "id": "q1",
  "question": "...",
  "answer": "...",
  "category": "precondition|computation|branching|guard|output|dependency|negative|boundary",
  "difficulty": "easy|medium|hard",
  "grounding": "Brief description of which code behavior this derives from"
}
\end{lstlisting}

\subsection{Probe sizing: \texttt{qa\_sizing.md}}
\label{app:prompt:qasize}

Invoked once per program to choose $n \in [30, 150]$ based on program-complexity signals (branches, cyclomatic-like structure, business-rule surface area). The returned count is passed as \texttt{\{num\_questions\}} to the Q\&A generation prompt above.
\begin{lstlisting}[language={}]
You are a software complexity analyst. Given the source code below, estimate how many behavioral Q&A pairs should be generated to adequately test a specification's coverage of this code.

Consider:
- Number of distinct behaviors, branches, and conditions
- Code complexity (nesting depth, number of functions/paragraphs)
- Number of input/output variables
- Edge cases and boundary conditions
- Number of user-visible screens, keys, fields, and validation rules
- Number of error handling paths and error messages

Aim for thorough coverage: it is better to generate too many questions than too few.

Return ONLY a JSON object: {"num_questions": <integer>}
The number should be between 30 and 150.

SOURCE CODE:
{source_code}
\end{lstlisting}

\subsection{Probe validation: \texttt{qa\_validation.md}}
\label{app:prompt:qaval}

The validator receives source code and generated question--answer pairs.
It is asked to check behavioural relevance, platform terminology,
infrastructure-only questions, answer correctness, ambiguity, and missing
qualifications. It returns a Boolean \texttt{valid} flag, an issue description,
and a suggested fix for each pair. The printed prompt does not request a
separate check of category or difficulty labels.
\begin{lstlisting}[language={}]
You are an adversarial reviewer for Q&A pairs generated from source code. Your job is to find problems.

For each Q&A pair, check:
1. Is the question about BEHAVIOR (good) or IMPLEMENTATION STRUCTURE (bad)?
   - Bad: mentions line numbers, variable names, function call order, types, syntax
2. Does the question or answer leak IMPLEMENTATION or PLATFORM terms?
   - Bad: COBOL keywords (PERFORM, MOVE, EVALUATE, INSPECT, PIC, COMP-3, 88-level, WORKING-STORAGE, LINKAGE)
   - Bad: CICS/middleware terms (COMMAREA, EIBCALEN, DFHRESP, BMS, SEND MAP, RECEIVE MAP, XCTL, SYNCPOINT, EXEC SQL, EXEC CICS)
   - Bad: file system internals (VSAM, KSDS, file status codes, READ INTO, DD names)
   - Bad: terminal/UI keybindings (F1, F3, F4, F5, F7, F8, F12, PF keys, AID keys, ENTER key, 3270)
     Instead use: "exit/go back", "confirm", "cancel", "submit", "page up/down"
   - Bad: internal names (WS- prefixed variables, paragraph names, copybook names)
   - Mark as invalid if any of these appear. The suggested_fix should rephrase using business-domain language.
3. Is the question about INFRASTRUCTURE rather than BUSINESS behavior?
   - Bad: program startup/shutdown banners ("END OF EXECUTION OF..."), diagnostic log messages, return codes, abend codes
   - Bad: internal file open/close sequences, queue initialization order, resource cleanup
   - Bad: questions whose answer is only observable in system logs, not in business output or user-facing screens
   - These are infrastructure artifacts, not behaviors a specification would describe. Mark as invalid.
4. Is the answer CORRECT given the source code?
5. Is the question AMBIGUOUS — could it reasonably be answered differently?
6. Is the answer PRECISE — or does it omit important conditions/qualifications?

SOURCE CODE:
{source_code}

Q&A PAIRS:
{qa_pairs_json}

Respond with a JSON array. Each element:
{
  "qa_id": "q1",
  "valid": true|false,
  "issue": "null or description of the problem",
  "suggested_fix": "null or corrected question/answer"
}
\end{lstlisting}

\subsection{Specification answerer: \texttt{judge.md}}
\label{app:prompt:judge}

Realises $J(B,q)$. The call receives only the specification and question,
and returns \texttt{answer}, \texttt{evidence}, and \texttt{confidence}.
The confidence field is \texttt{supported} when the specification supplies
enough information, or \texttt{uncertain} when information is missing or
ambiguous. This call sees neither the source code nor the reference answer.
The final agree/contradict/gap verdict is formed after answer comparison.
\begin{lstlisting}[language={}]
You are evaluating a specification. You will be given a question and you must answer it using ONLY the specification text below. Do NOT use any external knowledge or common sense.

If the specification contains multiple sections or programs, find the section most relevant to the question before answering. Do not combine information from unrelated sections.

SPECIFICATION:
{specification}

QUESTION:
{question}

Answer in this exact JSON format:
{
  "answer": "Your answer based solely on the specification",
  "evidence": "Quote or reference the specific part of the spec that supports your answer",
  "confidence": "supported|uncertain"
}

Confidence meanings:
- "supported": The spec contains enough information to answer this question
- "uncertain": The spec doesn't address this question, or is too ambiguous to answer
\end{lstlisting}

\subsection{Answer comparator: \texttt{comparator.md}}
\label{app:prompt:comparator}

Realises the answer-agreement decision used in the verdict function.
It compares the code-derived reference answer with the specification answer
and returns the Boolean field \texttt{agrees}.
The prompt allows different wording and some differences in specificity:
an exact error message or threshold may agree with a general statement
that an error is displayed or a limit is validated. Thus agreement does
not establish that every detail in the reference answer appears in the specification.
\begin{lstlisting}[language={}]
You compare two answers to the same question and determine if they agree on the BUSINESS BEHAVIOR being described. Answer A comes from legacy mainframe source code. Answer B comes from a modernization specification written for building a modern replacement.

They don't need to be word-for-word identical — just describe the same business outcome.

AGREE when:
- Both describe the same business behavior, even if Answer A uses legacy-specific terms (queue names, file operations, screen maps) and Answer B uses modern/abstract terms (service calls, data stores, user interface). The modernized description of the same behavior is agreement.
- One answer gives a specific value (an exact error message, a numeric threshold) and the other describes the behavior in general terms ("displays an error", "validates the limit").
- One answer includes extra business context or justification that the other omits.

DISAGREE only when:
- The answers describe DIFFERENT business outcomes (e.g., "saves the record" vs "discards the record").
- The answers give INCOMPATIBLE values (e.g., "maximum 10" vs "maximum 20").
- One answer affirms a behavior the other explicitly denies.

Respond with JSON: {"agrees": true} or {"agrees": false}
\end{lstlisting}

\subsection{Revision: \texttt{revision.md}}
\label{app:prompt:rev}

Realises $R$. The reviser receives the current specification and proposed
edits linked to requirement identifiers. It is instructed to preserve
unrelated requirements and returns the complete revised document.
Requirement identifiers locate the requested edits; the instruction does
not prevent changes elsewhere or establish the non-regression assumption.
\begin{lstlisting}[language={}]
You are a specification editor. You will be given:
1. An existing requirements specification
2. A list of improvement items derived from comparing the spec against the actual source code

Your job is to produce a REVISED specification that addresses the improvement items while preserving the overall structure, style, and requirement ID scheme of the original.

RULES:
- Preserve all existing requirement IDs (REQ-XXX-NNN). Do not renumber them.
- For FIX items: update the requirement text to match what the code actually does. Keep the same REQ ID.
- For ADD items: add new requirements using the next available ID in the appropriate section. Place them in the most relevant section.
- For REMOVE items: either delete the claim or add a qualifier like "[NOTE: not implemented in legacy system]" if the behavior is aspirational.
- Do NOT change requirements that are not listed in the improvement items.
- Preserve the markdown structure: headings, user stories, global preconditions, etc.
- When adding specific values (error messages, field names, thresholds), include them in the requirement text.
- Mark added or modified requirements with a trailing comment: <!-- REVISED: reason -->

EXISTING SPECIFICATION:
{spec_text}

IMPROVEMENT ITEMS:
{improvements_text}

Produce the complete revised specification. Include ALL sections from the original, not just the changed parts.
\end{lstlisting}

\end{document}